\documentclass[journal,a4paper,10pt]{IEEEtran}
\usepackage{amsmath,amsthm,amssymb,amsfonts,upref,cite,epsf,color}
\usepackage{pst-all,pstricks,pst-node,pst-text,pst-3d}
\usepackage{hyperref}
\usepackage{pstricks-add}
\usepackage{multido}
\usepackage{bm}
\usepackage{graphicx,epsfig,rotating} 
\usepackage{psfrag,texdraw,bm}
\usepackage{nomencl}
\usepackage{amsmath}
\usepackage{amssymb}
\usepackage{subfigure}
\usepackage{xspace}
\usepackage{dsfont}
\usepackage{pifont}
\usepackage{listings,color}   
\usepackage[T1]{fontenc}
\usepackage{textcomp}
\usepackage{framed}

\newcommand{\transp}[1]{{#1}^T} 

\newcommand\coefflen{p}
\newcommand{\measlen}{m} 
\newcommand{\samplesize}{N}
\newcommand{\sparsity}{s}

\newcommand{\obsidx}{k}
\newcommand{\dicidx}{l}

\newcommand\defeq{\triangleq}

\newcommand{\noisevec}{\mathbf{n}}
\newcommand{\noisemtx}{\mathbf{N}}
\newcommand{\coeffvec}{\mathbf{x}}
\newcommand{\coeffmtx}{\mathbf{X}}

\newcommand\vect[1]{\mathbf #1}

\newcommand{\vd}{\vect{d}}

\newcommand{\vx}{\vect{x}}  
\newcommand{\vy}{\vect{y}}  
\newcommand{\vz}{\vect{z}}

\newcommand{\mD}{\vect{D}}

\newcommand{\mN}{\vect{N}}

\newcommand{\mX}{\vect{X}}
\newcommand{\mY}{\vect{Y}}

\newcommand{\minimaxrisk}{\varepsilon^{*}}

\usepackage[applemac]{inputenc}

\def\baselinestretch{1}
\renewcommand{\baselinestretch}{1.6}\small\normalsize
\def \expect {{\rm E} }
\def \prob {{\rm P} }
\def \twiddle[#1] {e^{-j \frac{2 \pi}{N}  #1 }}
\def \twiddleneg[#1] {e^{j \frac{2 \pi}{N}  #1 }}

\DeclareMathOperator{\supp}{supp}

\DeclareMathOperator*{\rank}{rank}

\DeclareMathOperator*{\argmin}{argmin}

\DeclareMathOperator*{\trace}{Tr}

\newtheorem{theorem}{Theorem}[section]

\newtheorem{lemma}[theorem]{Lemma}

\newtheorem{corollary}[theorem]{Corollary}
\newtheorem{algorithm}{Algorithm}

\def\ML_est{\hat{\mathbf{x}}_{\text{ML}}}
\newcommand{\CRBfull}{{C}ram\'{e}r--{R}ao bound\xspace}

\newcommand{\be}{\begin{equation}}
\newcommand{\ee}{\end{equation}}

\newcommand{\rmv}{\hspace*{-.2mm}}

\newcommand{\condmiyandlgivenx}{I(\mathbf{Y};\dicidx| \mathbf{T}(\mathbf{X}))}

\allowdisplaybreaks

\makeatother

\usepackage[onehalfspacing]{setspace}

\onecolumn
\begin{document}

\title{On the Minimax Risk of Dictionary Learning}
%
\author{\emph{Alexander Jung$\rmv^{a}\rmv$, Yonina C. Eldar$^{b}\rmv\rmv$, Norbert G\"{o}rtz$^{a}\rmv\rmv$
} 

{\normalsize $^a$ Institute of Telecommunications, Vienna University of Technology, Austria; ajung@nt.tuwien.ac.at\\[-0.5mm]}
{
\normalsize $^b$Technion---Israel Institute of Technology, Israel; e-mail: yonina@ee.technion.ac.il}
\thanks{Parts of this work were previously presented at the 22nd European Signal Processing Conference, Lisbon, PT, Sept.\ 2014.}

}

\maketitle




\begin{abstract}
We consider the problem of learning a dictionary matrix from a number of observed signals, which are assumed to be generated via a linear model with a common underlying dictionary. In particular, we derive lower bounds on the minimum achievable {worst case} mean squared error (MSE), regardless of computational complexity of the dictionary learning (DL) schemes. By casting DL as a classical (or frequentist) estimation problem, the lower bounds on the {worst case} MSE are derived by following an established information-theoretic {approach to minimax estimation}. 
The main conceptual contribution of this paper is the adaption of the information-theoretic approach to minimax estimation for the DL problem {in order to derive lower bounds on the worst case MSE of any DL scheme. We derive three different lower bounds applying to different generative models for the observed signals. The first bound applies to a wide range of models, it only requires the 
existence of a covariance matrix of the (unknown) underlying coefficient vector. By specializing this bound to the case of sparse coefficient distributions, and assuming the true dictionary 
satisfies the restricted isometry property, we obtain a lower bound on the worst case MSE of DL schemes in terms of a signal to noise ratio (SNR). The third bound applies to a more restrictive subclass of coefficient distributions by requiring the non-zero coefficients to be Gaussian. While, compared with the previous two bounds, the applicability of this final bound is the most limited it is the tightest of the three bounds in the low SNR regime}. A particular use of our lower bounds is the derivation of necessary conditions on the required number of observations (sample size) such that DL is feasible, i.e., {accurate DL schemes might exist}. 
By comparing these necessary conditions with sufficient conditions on the sample size such that a particular DL scheme is successful, we are able to characterize the regimes where those algorithms 
are optimal (or {possibly} not) in terms of required sample size. 
\end{abstract}

\section{Introduction}
\label{sec_intro} 

\vspace{-.2mm}

According to \cite{ciscozettabyte2014}, the worldwide internet traffic in $2016$ will exceed the Zettabyte threshold.\footnote{One Zettabyte equals $10^{21}$ bytes.} 
In view of the pervasive massive datasets generated at an ever increasing speed \cite{DataDelEconomist,DataEverywhere}, it is mandatory to be able to extract relevant information out of the observed data. 
A recent approach to this challenge, which has proven extremely useful for a wide range of applications, is \emph{sparsity} and the related theory of \emph{compressed sensing} (CS) \cite{Don06,Can06,EldarKutyniokCS}. 
In our context, sparsity means that the observed signals can be represented by a linear combination of a small number of prototype functions or atoms. In {many} applications the set of atoms is pre-specified and stored in a dictionary matrix. However, in some applications it might be necessary or beneficial to adaptively determine a dictionary based on the observations \cite{argemo13,protter09,Peyre2009}. The task of adaptively determining the underlying dictionary matrix is referred to as \emph{dictionary learning} (DL). 
DL {has been considered} for a wide range of applications{, such as image processing \cite{OlshausenField97,TosicFrossard2011,Turkan2011,MairalBach2009,MairalEladSapiro2008}, blind source separation \cite{ZibPearl2001}, sparse principal component analysis \cite{JenObozBachSSPCA}, and more}. 

In this paper, we consider observing $\samplesize$ signals $\mathbf{y}_{\obsidx} \in \mathbb{R}^{\measlen}$ generated via a {fixed (but unknown)} underlying dictionary 
$\mathbf{D} \! \in\! \mathbb{R}^{\measlen\times \coefflen}$ (which we would like to estimate). More precisely, the observations $\mathbf{y}_{\obsidx}$ are modeled as noisy linear combinations
\begin{equation} 
\label{equ_single_linear_model}
\mathbf{y}_{\obsidx} = \mathbf{D} \coeffvec_{\obsidx} + \noisevec_{\obsidx},
\end{equation} 
{where $\noisevec_{\obsidx}$ is assumed to be zero-mean with i.i.d. components of variance $\sigma^{2}$.} 
{To formalize the estimation problem underlying DL, 
we assume the coefficient vectors $\vx_{\obsidx}$ to be zero-mean random vectors with finite covariance matrix ${\bm \Sigma}_{x}$. We highlight that our first main result, i.e., Theorem \ref{thm_main_result} applies to a very wide class of coefficient distributions since it only requires a finite covariance matrix ${\bm \Sigma}_{x}$. In particular, Theorem \ref{thm_main_result} also applies to non-sparse random coefficient vectors. However, the main focus of our paper (in particular, for Corollary \ref{cor_main_result_sparse_vectors} and Theorem \ref{thm_main_result_sparse_coeff}) will be on distributions such that the coefficient vector $\vx_{\obsidx}$ is strictly $\sparsity$-sparse with probability one.} 
In this work, we analyze the difficulty inherent to the problem of estimating the true dictionary $\mathbf{D} \in \mathbb{R}^{\measlen \times \coefflen}$, which is deterministic but unknown, from the measurements  $\vy_{\obsidx}$, which {are generated according to the linear model \eqref{equ_single_linear_model}.}

{
If we stack the observations $\mathbf{y}_{\obsidx}$, for $\obsidx=1,\ldots,\samplesize$, column-wise into the data matrix $\mathbf{Y} \in \mathbb{R}^{\measlen \times \samplesize}$, one can cast DL as a matrix factorization problem \cite{BachMairalPonce2008}. 
Given the data matrix $\mathbf{Y}$, we aim to find a dictionary matrix $\mathbf{D} \in \mathbb{R}^{\measlen \times \coefflen}$ such that 
\begin{equation}
\label{equ_factorizing_expansion_noise}
\mathbf{Y} = \mathbf{D} \coeffmtx + \noisemtx
\end{equation} 
where the column sparse matrix $\coeffmtx \in \mathbb{R}^{\coefflen \times \samplesize}$ contains in its $\obsidx$th column the sparse expansion coefficients $\mathbf{x}_{\obsidx}$ 
of the signal $\mathbf{y}_{\obsidx}$. The noise matrix $\noisemtx=\big(\noisevec_{1},\ldots,\noisevec_{\samplesize}\big) \in \mathbb{R}^{\measlen \times \samplesize}$ accounts for small modeling and measurement errors.} 

\paragraph{{Prior Art}} 
A plethora of DL methods have been proposed and analyzed in the literature (e.g., \cite{Schnass2014,KSVD,Jenatton2012,GribonvalSchnass2010,YaghDicLearn2009,BilinAMPDer,BilinAMPApp,spwawr12,argemo13,aganjaneta13}). 
In a Bayesian setting, i.e., modeling the dictionary as random with a known prior distribution, {the authors of} \cite{BilinAMPDer,BilinAMPApp,Krzakala13} devise a variant of the \emph{approximate message passing} scheme \cite{AMP2009PNAS} to the DL problem. 
The authors of \cite{KSVD,Jenatton2012,GribonvalSchnass2010,YaghDicLearn2009,Mairal2009ICML} model the dictionary as non-random and 
estimate the dictionary by solving the (non-convex) optimization problem
\begin{equation}
\label{equ_minimization_ell_1_dic_learning}
\min_{{\mathbf{D} \in \mathcal{D},\coeffmtx \in \mathbb{R}^{\coefflen \times \samplesize}}}  \| \mathbf{Y} \!-\! \mathbf{D} \coeffmtx \|_{F}^{2} + \lambda \|\coeffmtx\|_{1},
\end{equation}
where $\|\coeffmtx\|_{1} \defeq \sum_{k,l} |\coeffmtx_{k,l}|$ {and $\mathcal{D} \subseteq \mathbb{R}^{\measlen \times \coefflen}$ denotes a constraint set, e.g., requiring}
{the columns of the learned dictionary to have unit norm.} The term $\lambda \|\coeffmtx\|_{1}$ (with sufficiently large $\lambda$) in the objective \eqref{equ_minimization_ell_1_dic_learning} 
{enforces the columns of the coefficient matrix $\coeffmtx$ to be (approximately) sparse}.

Assuming the true dictionary {$\mD\!\in\!\mathbb{R}^{\measlen \times \coefflen}$} deterministic but unknown (its size {$\coefflen$} however is known) and the observations $\mathbf{y}_{\obsidx}$ are i.i.d. according to the 
model {\eqref{equ_single_linear_model}}, the authors of \cite{KSVD,Jenatton2012,GribonvalSchnass2010} {provide} upper bounds on the distance {between} the generating 
dictionary {$\mD$} and the closest local minimum of \eqref{equ_minimization_ell_1_dic_learning}. For the square (i.e., $\coefflen = \measlen$) and noiseless ($\mathbf{N}=\mathbf{0}$) setting, \cite{GribonvalSchnass2010} showed that $\samplesize = \mathcal{O}( \coefflen \log( \coefflen))$ observations suffice to guarantee that the dictionary is a local minimum of \eqref{equ_minimization_ell_1_dic_learning}. Using the same setting (square dictionary and noiseless measurements), \cite{spwawr12} proved the scaling $\samplesize = \mathcal{O}(\coefflen \log( \coefflen))$, for arbitrary sparsity level, to be actually sufficient such that the dictionary {matrix} can be recovered {perfectly} from the {measurements $\vy_{\obsidx}$}.\footnote{With high probability and up to scaling and permutations of the {dictionary} columns.} {Our analysis, in contrast, takes measurement noise into account and yields lower bounds on the required sample size in terms of SNR.} 
{While the results on the square-dictionary and noiseless case are theoretically important, their practical relevance is limited.} 
Considering the practically more relevant case of an overcomplete ($\coefflen > \measlen$) dictionary $\mathbf{D}$ and noisy measurements ($\mathbf{N} \neq \mathbf{0}$), the authors of \cite{Jenatton2012} show that a sample size of $\samplesize = \mathcal{O}(\coefflen^{3} \measlen)$ i.i.d.\ {measurements} $\vy_{\obsidx}$ suffices for the existence of a local minimum of the cost function in \eqref{equ_minimization_ell_1_dic_learning} which is close to 
the true dictionary {$\mD$}. 
 
By contrast to methods based on solving \eqref{equ_minimization_ell_1_dic_learning}, a recent line of work \cite{spwawr12,argemo13,aganjaneta13} presents DL methods based on {(graph-)clustering} techniques. In particular, the set of observed samples $\vy_{\obsidx}$ is clustered such that the elements within each cluster share a single generating column $\mathbf{d}_{j}$ of the underlying dictionary. 
The authors of \cite{aganjaneta13} show that a sample size $\samplesize = \mathcal{O}(\coefflen^{2} \log \coefflen)$ suffices for their clustering-based method to accurately recover the true underlying dictionary. 
However, this result applies only for sufficiently incoherent dictionaries $\mathbf{D}$ and for the case of vanishing sparsity rate, i.e., {$\sparsity/\coefflen \rightarrow 0$}. {The 
scaling of the required sample size with the square of the number $\coefflen$ of dictionary columns (neglecting logarithmic terms) is also predicted by our bounds. What sets our 
work apart from \cite{aganjaneta13} is that we state our results in a non-asymptotic setting, i.e., our bounds can be evaluated for any given number $\coefflen$ of dictionary atoms, 
dimension $\measlen$ of observed signals and nominal sparsity level $\sparsity$.}


{Although numerous DL schemes have been proposed and analyzed, existing analyses typically yield sufficient conditions (e.g., on the sample size $\samplesize$) such that DL is feasible. In contrast,  necessary conditions which apply to any DL scheme (irrespective of computational complexity) are far more limited. We are only aware}
of a single fundamental result that applies to a Bernoulli-Gauss prior for the coefficient vectors $\vx_{k}$ in \eqref{equ_single_linear_model}: This result, also known as the ``coupon collector phenomenon'' \cite{spwawr12}, states that in order to have every column $\mathbf{d}_{j}$ of the dictionary contributing in at least one observed signal (i.e., the corresponding entry $x_{k,j}$ of the coefficient vector in \eqref{equ_single_linear_model} is non-zero) the sample size has to scale linearly with 
$(1/\theta) \log \coefflen$ where $\theta$ denotes the probability $\prob \{ x_{k,j} \neq 0 \}$. For the choice $\theta = \sparsity/\coefflen$, which yields {$\sparsity$-sparse} coefficient vectors with high probability, this requirement {effectively} becomes {$\samplesize \geq c_{1} (\coefflen/\sparsity) \log \coefflen$, with some absolute constant $c_{1}$}. 

\paragraph{{Contribution}} 
In this paper we contribute to the understanding of necessary conditions or fundamental recovery thresholds for DL, by deriving lower bounds on the minimax risk for the DL problem. We define the risk incurred by a DL scheme as the mean squared error (MSE) using the Frobenius norm of the deviation from the true underlying dictionary. Since the minimax risk is defined as the minimum achievable worst case MSE, our lower bounds apply to the worst case MSE of any algorithm, regardless of its computational complexity. This paper seems to contain the first analysis that targets directly the fundamental limits on the achievable MSE of any DL method. 

For the derivation of the lower bounds, we apply an established information-theoretic approach (cf.\ Section \ref{SecProblemFormulation}) to 
minimax estimation, which is based on reducing a specific multiple hypothesis problem to minimax estimation of the dictionary {matrix}. 
Although this {information-theoretic} approach has been successfully applied to several other (sparse) {minimax} estimation problems \cite{WangWain2010,Wain2009TIT,CandesDavenport2013,Santhanam2012,CaiZhouSparseCov}, the adaptation  of this method to the problem of DL seems to be new. The lower bounds on the minimax risk give insight into the dependencies of the achievable worst case MSE on the model parameters, i.e., the sparsity $\sparsity$, the dictionary size $\coefflen$, the dimension $\measlen$ of the observed signal {and the SNR.} 
Our lower bounds on the minimax risk have direct implications on the required sample size of accurate DL schemes. 
{
In particular our analysis reveals that, for a sufficiently incoherent underlying dictionary, the minimax risk of DL is lower bounded by $c_{1}  \coefflen^2/({\rm SNR} \samplesize)$, where $c_{1}$ is some 
absolute constant. Thus, for a vanishing minimax risk it is necessary for the sample size $\samplesize$ to scale linearly with the square of the number $\coefflen$ of dictionary columns and 
inversely with the SNR.} 
Finally, by comparing our lower bounds {(on minimax risk and sample size)} with the performance guarantees of existing learning schemes, we can test if these methods {perform close to optimal}.

{A recent work on the sample complexity of dictionary learning \cite{VainManBruck11} presented upper bounds on the sample size such that the (expected) performance of an ideal learning scheme is close to its empirical performance observed when applied to the observed samples. 
While the authors of \cite{VainManBruck11} measure the quality of the estimate $\widehat{\mathbf{D}}$ via the residual error obtained when sparsely approximating the observed vectors $\vy_{\obsidx}$, 
we use a different risk measure based on the squared Frobenius norm of the deviation from the true underlying dictionary. Clearly, these two risk measures 
are related. Indeed, if the Frobenius norm $\| \widehat{\mathbf{D}} \!-\! \mathbf{D}\|_{\text{F}}$ is small, we can also expect that any sparse linear combination $\mathbf{D} \mathbf{x}$ using the dictionary $\mathbf{D}$ can also be well represented by a sparse linear combination $\widehat{\mathbf{D}}\mathbf{x}'$ using $\widehat{\mathbf{D}}$. 
Our results are somewhat complementary to the upper bounds in \cite{VainManBruck11} in that they yield lower bounds on the required sample size such that there may exist accurate learning schemes (regardless of computational complexity).
}


The remainder of this paper is organized as follows: 
We introduce the minimax risk of DL and the information-theoretic method for lower bounding it in Section \ref{SecProblemFormulation}. Lower bounds on the minimax risk for DL are presented in Section \ref{sec_main_result}. We also put our bounds into perspective by comparing their 
implications to the available performance guarantees of some DL schemes. Detailed proofs of the main results are contained in Section \ref{proof_architecture_main_result}. 

Throughout the paper, we use the following notation: 
Given a natural number $k \in \mathbb{N}$, we define the set $[k] \triangleq \{1,\ldots,k\}$. For a matrix $\mathbf{A} \in \mathbb{R}^{\measlen \times \coefflen}$, we denote its Frobenius norm and its spectral norm by 
$\| \mathbf{A} \|_{\text{F}} \triangleq \sqrt{\trace \{\mathbf{A}\mathbf{A}^{T}\}}$ and $\| \mathbf{A} \|_{\text{2}}$, respectively. {The open (Frobenius-norm) ball of radius $r>0$ and center $\mathbf{D} \in \mathbb{R}^{\measlen \times \coefflen}$ is denoted $\mathcal{B}(\mathbf{D},r) \defeq \{ \mathbf{D}' \in \mathbb{R}^{\measlen \times \coefflen} : \| \mathbf{D} - \mathbf{D}' \|_{\rm F} < r \}$.}
For a square matrix $\mathbf{A}$, the vector containing the elements along the diagonal of $\mathbf{A}$ is denoted ${\rm{diag}}\{\mathbf{A}\}$. Analogously, given a vector $\mathbf{a}$, we denote by ${\rm{diag}} \{ \mathbf{a} \}$ 
the diagonal matrix whose diagonal is obtained from $\mathbf{a}$. 
The $k$th column of the identity matrix is denoted $\mathbf{e}_{k}$. 
For a matrix $\coeffmtx \in \mathbb{R}^{\coefflen \times \samplesize}$, we denote by $\supp(\coeffmtx)$ the $\samplesize$-tuple $\big(\supp(\coeffvec_{1}),\ldots,\supp(\coeffvec_{\samplesize})\big)$ of subsets given by 
concatenating the supports $\supp(\coeffvec_{\obsidx})$ of the columns $\coeffvec_{\obsidx}$ of the matrix $\coeffmtx$. 
The complementary Kronecker delta is denoted $\bar{\delta}_{l,l'}$, i.e., $\bar{\delta}_{l,l'} = 0$ if $l=l'$ and equal to one otherwise. 
We denote by $\mathbf{0}$ the vector or matrix with all entries equal to $0$. 
The determinant of a square matrix $\mathbf{C}$ is denoted $|\mathbf{C}|$. {The identity matrix is written as $\mathbf{I}$ or $\mathbf{I}_{d}$ when the dimension $d \times d$ is not clear from the context.} 
Given a positive semidefinite (psd) matrix $\mathbf{C}$, we write its smallest eigenvalue as $\lambda_{\text{min}}(\mathbf{C})$. The natural and binary logarithm of a number $b$ are denoted $\log(b)$ and $\log_{2}(b)$, respectively. For two sequences $g(\samplesize)$ and $f(\samplesize)$, indexed by the natural number $\samplesize$, we 
write $g = \mathcal{O}(f)$ and $g = \Theta(f)$ if, respectively, $g(\samplesize) \leq C' f(\samplesize)$ and $g(\samplesize) \geq C'' f(\samplesize)$ for some constants $C',C'' > 0$. {If $g(\samplesize)/f(\samplesize) \rightarrow 0$, we write $g=o(f)$. We denote by $\expect_{\mathbf{X}}{f(\mathbf{X})}$ the expectation of the function $f(\mX)$ of the random vector (or matrix) $\mX$.}

\vspace{-.7mm}

\section{Problem Formulation}
\label{SecProblemFormulation}

\subsection{Basic Setup}
For our analysis we assume the observations $\vy_{\obsidx}$ are i.i.d. realizations according to the random linear model 
\begin{equation}
\label{equ_linear_model}
\vy = \mathbf{D} \mathbf{x} + \noisevec. 
\end{equation} 
Thus, the vectors $\vy_{\obsidx}$, $\coeffvec_{\obsidx}$ and $\noisevec_{\obsidx}$, for $\obsidx=1,\ldots,\samplesize$, in \eqref{equ_single_linear_model} 
are i.i.d. realizations of the random vectors $\vy$, $\coeffvec$ and $\noisevec$ in \eqref{equ_linear_model}. 
Here, the matrix $\mathbf{D}\!\in\! \mathbb{R}^{\measlen \times \coefflen}$, with $\coefflen \geq \measlen$, represents the deterministic but unknown underlying dictionary, 
whose columns are the building blocks of the observed signals $\vy_{\obsidx}$. 
The vector $\coeffvec$ represents zero mean random expansion coefficients, whose distribution is assumed to be known. 
Our analysis applies to a wide class of distributions. In fact, we only 
require the existence of the coveriance matrix 
\begin{equation} 
\label{equ_corr_coeffs_exists}
\mathbf{\Sigma}_{x} \defeq \expect_{{\vx}} \big\{ \mathbf{x} \mathbf{x}^{T}  \}.
\end{equation} 
The effect of modeling and measurement errors are captured by the noise vector $\noisevec$, which is assumed independent of $\coeffvec$ and is white Gaussian noise (AWGN) with zero mean and known variance $\sigma^{2}$. When combined with a sparsity enhancing prior on $\coeffvec$, the linear model \eqref{equ_linear_model} reduces to the sparse linear model (SLM) \cite{RKHSSLGMIT2012}, which is the workhorse of CS \cite{CandesCSTutorial,BaraniukCSTutorial,EldarKutyniokCS}. 
However, while the works on the SLM typically assume the dictionary $\mathbf{D}$ in \eqref{equ_linear_model} perfectly known, we consider the situation where $\mathbf{D}$ is unknown. 

In what follows, we assume the columns of the dictionary $\mathbf{D}$ to be normalized, i.e., 
\begin{equation}
\label{equ_def_dictionary_set_normalized_cols}
\mathbf{D} \in \mathcal{D} \triangleq \{ \mathbf{B} \in \mathbb{R}^{\measlen \times \coefflen} | \mathbf{e}_{k}^{T} \transp{\mathbf{B}}\mathbf{B} \mathbf{e}_{k}= 1 \mbox{, for all } k\in [\coefflen] \}. 
\end{equation}
The set $\mathcal{D}$ is known as the \emph{oblique manifold} \cite{Jenatton2012,AbsilMahonySepulchre,Absil2006}. {For fixed problem dimensions $\coefflen$, $\measlen$ and $\sparsity$, requiring \eqref{equ_def_dictionary_set_normalized_cols} effectively amounts to identifying SNR with the quantity $\| \mathbf{\Sigma}_{x} \|_{2}/\sigma^{2}$.} 
Our analysis is local in the sense that we consider the true dictionary $\mathbf{D}$ to belong to a small neighborhood, i.e., 
\begin{equation} 
\label{eq_dic_belongs_to_small_ball}
\mathbf{D} \in {\mathcal{X}(\mathbf{D}_{0},r) \defeq\mathcal{B}(\mathbf{D}_{0},r) \cap \mathcal{D} = \{ \mathbf{D}' \in \mathcal{D}: \| \mathbf{D}' \!-\! \mathbf{D}_{0} \|_{\rm{F}} < r \}}
\end{equation} 
with a fixed and known ``reference dictionary'' $\mathbf{D}_{0} \in \mathcal{D}$ {and known radius\footnote{{Considering only values not exceeding $2\sqrt{\coefflen}$ for the radius $r$ in \eqref{eq_dic_belongs_to_small_ball} is reasonable since for any radius $r>2\sqrt{\coefflen}$ we would obtain $\mathcal{X}(\mathbf{D}_{0},r)=\mathcal{D}$ yielding the global DL problem}.} $r\leq 2\sqrt{\coefflen}$}. {This local analysis avoids ambiguity issues {(which we discuss below)} that are intrinsic to DL. However, the 
lower bounds on the minimax risk derived on the locality constraint \eqref{eq_dic_belongs_to_small_ball} trivially also apply to the global DL problem, i.e., where we only require \eqref{equ_def_dictionary_set_normalized_cols}.} 

\subsection{The minimax risk}

We will investigate the fundamental limits on the accuracy achievable by any DL scheme, irrespective of its computational complexity. By a DL scheme, we mean an estimator $\widehat{\mathbf{D}}(\cdot)$ which maps the observation $\mathbf{Y}\!=\!\big(\vy_{1},\ldots,\vy_{\samplesize}\big)$ to an estimate $\widehat{\mathbf{D}}(\mathbf{Y})$ of the true underlying dictionary $\mathbf{D}$.
The accuracy of a given learning method will be measured via the MSE $\expect_{{\mY}} \{ \| \widehat{\mathbf{D}}(\mY) - \mathbf{D} \|^{2}_{\text{F}} \}$, which 
is the expected squared distance of the estimate $\widehat{\mathbf{D}}(\mY)$  from the true dictionary, measured in Frobenius norm. 
Note that the MSE of a given learning scheme $\widehat{\mathbf{D}}(\mY)$ depends on the true underlying dictionary $\mathbf{D}$, {which is fixed but unknown.} 
Therefore, the MSE cannot be minimized uniformly for all $\mathbf{D}$ \cite{RethinkingBiasedEldar}. However, for a given estimator $\widehat{\mathbf{D}}(\cdot)$, a reasonable 
performance measure is the worst case MSE $\sup_{\mathbf{D} \in {\mathcal{X}(\mathbf{D}_{0},r)}} \expect_{{\mY}} \{ \| \widehat{\mathbf{D}}(\mY) \!-\! \mathbf{D} \|_{\text{F}}^{2} \}$ \cite{LC}. The optimum estimator under this criterion has smallest worst case MSE among all possible estimators. This smallest worst case MSE ({referred to as minimax risk}) is 
an intrinsic property of the estimation problem and does not depend on a specific estimator. {Let us highlight that the minimax risk is defined here for a fixed and known distribution of the 
coefficient vector $\vx_{\obsidx}$ in \eqref{equ_single_linear_model}. In what follows, we derive three different lower bounds on the minimax risk by considering different types of coefficient distributions.} 

Concretely, the minimax risk $\minimaxrisk$ for the problem of learning the dictionary $\mathbf{D}$ based on the observation of $\samplesize$ i.i.d. observations $\vy_{\obsidx}$, distributed according to the model \eqref{equ_linear_model}, is 
\begin{equation}
\label{equ_def_minimax_problem}
\minimaxrisk \triangleq \inf_{\widehat{\mathbf{D}}} \sup_{\mathbf{D} \in {\mathcal{X}}(\mathbf{D}_{0},r)}  \expect_{{\mY}}  \{ \| \widehat{\mathbf{D}}(\mathbf{Y}) \!-\! \mathbf{D} \|_{\text{F}}^{2} \}.
\end{equation} 
In general, the minimax risk $\minimaxrisk$ depends on the sample size $\samplesize$, the dimension $\measlen$ of the observed signals, the 
number $\coefflen$ of dictionary elements, the sparsity degree $\sparsity$ and the noise variance $\sigma^{2}$. 
For the sake of light notation, we will not make this dependence explicit. 

{Note that while, at first sight, the locality assumption \eqref{eq_dic_belongs_to_small_ball} may suggest that our analysis yields weaker results than for the case of not having this locality assumption, the opposite is actually true. Indeed, 
our lower bounds on the minimax risk predict that even under the additional a-priori knowledge that the true dictionary belongs to the (small) neighborhood of a known reference dictionary $\mD_{0}$, the 
minimax risk is lower bounded by a strictly positive number which, for a sufficiently large sample size, does not depend on the size of the neighborhood at all. Also, from the definition \eqref{equ_def_minimax_problem} it is obvious that any lower bound on the minimax risk $\minimaxrisk$ under 
the locality constraint \eqref{eq_dic_belongs_to_small_ball} is simultaneously a lower bound on the minimax risk for global DL, which is obtained from \eqref{equ_def_minimax_problem} by replacing the constraint 
$\mathbf{D} \in \mathcal{X}(\mathbf{D}_{0},r)$ in the inner maximization with the constraint $\mathbf{D} \in \mathcal{D}$.} 

The minimax problem \eqref{equ_def_minimax_problem} typically cannot be solved in closed-form. Instead of trying to exactly solve \eqref{equ_def_minimax_problem} and determine $\minimaxrisk$, we will derive 
lower bounds on $\minimaxrisk$ by adapting an established information-theoretic methodology (cf., e.g., \cite{BinYuFestschrift1997,WangWain2010,CandesDavenport2013}) to the DL problem.
Having a lower bound on the minimax risk $\minimaxrisk$ allows to asses the performance of a given DL scheme. In particular, 
if the worst case MSE of a given scheme is close to the lower bound, then there is no point in searching for alternative schemes with substantially better 
performance. {Let us highlight that our bounds apply to any DL scheme, regardless of its computational complexity. In particular, these 
bounds apply also to DL methods which do not exploit neither the knowledge of the sparse coefficient distribution nor of the noise variance.}

\subsection{Information-theoretic lower bounds on the minimax risk}
\label{sec_inf_th_lower_bounds_minimax_risk}

A principled approach \cite{BinYuFestschrift1997,WangWain2010,CandesDavenport2013} to lower bounding the minimax risk $\minimaxrisk$ of a general estimation problem is based on reducing a specific multiple hypothesis testing problem to minimax estimation of the dictionary $\mathbf{D}$. More precisely, if there exists an estimator with small worst case MSE, then this estimator can be used to solve a hypothesis testing problem. However, using Fano's inequality, there is a fundamental limit on the error probability for the hypothesis testing problem. This limit induces a lower bound on the worst case MSE of any estimator, i.e., on the minimax risk. Let us now outline the details of the method. 

First, within this approach one assumes that the true dictionary $\mathbf{D}$ in \eqref{equ_linear_model} is taken uniformly at random (u.a.r.) from a finite subset $\mathcal{D}_{0} \triangleq \{ \mathbf{D}_{\dicidx} \}_{\dicidx \in [L]} \subseteq {\mathcal{X}}(\mathbf{D}_{0},r)$ for some $L \in \mathbb{N}$ (cf.\ Fig. \ref{fig_info_theor_proof_sketch}). 
This subset $\mathcal{D}_{0}$ is constructed such that (i) any two distinct dictionaries $\mathbf{D}_{\dicidx},\mathbf{D}_{\dicidx'} \in \mathcal{D}_{0}$ are separated 
by at least $\sqrt{8 \varepsilon}$, i.e., $\| \mathbf{D}_{\dicidx} - \mathbf{D}_{\dicidx'} \|_{\text{F}} \geq \sqrt{8 \varepsilon}$ and (ii) it is hard to detect the true dictionary $\mathbf{D}$, drawn u.a.r. out of $\mathcal{D}_{0}$, based on observing $\mathbf{Y}$. 
The existence of such a set $\mathcal{D}_{0}$ yields a relation between the sample size $\samplesize$ and the remaining model parameters, i.e., $\measlen$, $\coefflen$, $\sparsity$, $\sigma$ which has to be satisfied such that {at least one} estimator with minimax-risk not exceeding $\varepsilon$ may exist.
\begin{figure}[t]
\vspace{-1mm}
\centering
\psfrag{SNR}[c][c][.9]{\uput{3.4mm}[270]{0}{\hspace{0mm}SNR [dB]}}
\psfrag{radius}[c][c][.9]{\hspace*{-1mm}$\sqrt{2 \varepsilon}$}
\psfrag{distance}[c][c][.9]{\hspace*{-1mm}$\sqrt{8 \varepsilon}$}
\psfrag{DicSet}[c][c][.9]{\hspace*{-1mm}$\mathcal{D}$}
\psfrag{D1}[c][c][.9]{\hspace*{2mm}$\mathbf{D}_{1}$}
\psfrag{D2}[c][c][.9]{\hspace*{2mm}$\mathbf{D}_{2}$}
\psfrag{D3}[c][c][.9]{\hspace*{2mm}$\mathbf{D}_{3}$}
\psfrag{D4}[c][c][.9]{\hspace*{2mm}$\mathbf{D}_{4}$}
\psfrag{hatD}[c][c][.9]{\hspace*{4mm}$\widehat{\mathbf{D}}(\mathbf{Y})$}
\psfrag{x_m20}[c][c][.9]{\uput{0.3mm}[270]{0}{$-20$}}
\psfrag{x_m_10}[c][c][.9]{\uput{0.3mm}[270]{0}{$-10$}}
\centering
\hspace*{-0mm}\includegraphics[height=7cm]{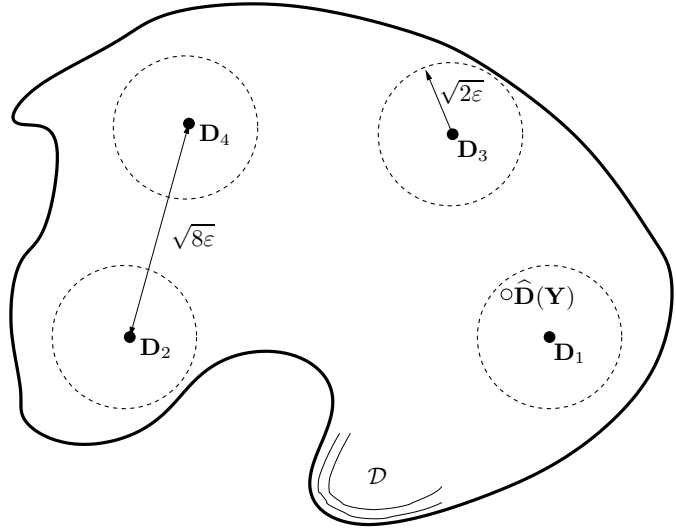}
\vspace*{-.5mm}
\renewcommand{\baselinestretch}{1.2}\small\normalsize
  \caption{{A finite ensemble $\mathcal{D}_{0} \!=\! \{ \mathbf{D}_{\dicidx} \}_{\dicidx \in [L]}$ containing $L=4$ dictionaries used for deriving a lower bound $\minimaxrisk \!\geq \varepsilon$ on the minimax risk $\minimaxrisk$ (cf.\ \eqref{equ_def_minimax_problem}). {For the true dictionary $\mathbf{D}=\mathbf{D}_{1}$, we also depicted a typical} realization of an estimator $\widehat{\mathbf{D}}$ achieving the minimax risk.}} 
\label{fig_info_theor_proof_sketch}
\vspace*{3.5mm}
\end{figure}

In order to find a lower bound {$\minimaxrisk\geq \varepsilon$} on the minimax risk $\minimaxrisk$ (cf.\ \eqref{equ_def_minimax_problem}), we hypothesize the existence of an estimator $\widehat{\mathbf{D}}(\mathbf{Y})$ achieving the minimax risk in \eqref{equ_def_minimax_problem}. Then, the 
minimum distance detector 
\begin{equation}
\argmin_{\mathbf{D}' \in \mathcal{D}_{0}} \| \widehat{\mathbf{D}}(\mathbf{Y})\!-\!\mathbf{D}' \|_{\text{F}} 
\end{equation} 
recovers the correct dictionary $\mathbf{D} \in \mathcal{D}_{0}$ if $\widehat{\mathbf{D}}(\mathbf{Y})$ belongs to the {open} ball $\mathcal{B}(\mathbf{D},\sqrt{2 \varepsilon})$ (indicated by the dashed circles in Fig. \ref{fig_info_theor_proof_sketch}) centered at $\mathbf{D}$ and with radius $\sqrt{2 \varepsilon}$. 
The information-theoretic method \cite{Wain2009TIT,WangWain2010,BinYuFestschrift1997} of lower bounding the minimax risk $\minimaxrisk$ consists then in relating, via Fano's inequality \cite[Ch.\ 2]{coverthomas}, the error probability $\prob \big\{ \widehat{\mathbf{D}}(\mathbf{Y}) \notin \mathcal{B}(\mathbf{D},\sqrt{2 \varepsilon}) \big\}$ to the mutual information (MI) between the observation $\mathbf{Y} \!=\! \big(\mathbf{y}_{1},\ldots,\mathbf{y}_{\samplesize}\big)$ and the dictionary $\mathbf{D}$ in \eqref{equ_linear_model}, which is assumed to be drawn u.a.r. out of $\mathcal{D}_{0}$. 

\begin{figure}[t]
\vspace{-1mm}
\centering
\psfrag{SNR}[c][c][.9]{\uput{3.4mm}[270]{0}{\hspace{0mm}SNR [dB]}}
\psfrag{uar}[c][c][.9]{\hspace*{-1mm}{u.a.r.}}
\psfrag{Source}[c][c][.9]{\hspace*{-1mm}{Source}}
\psfrag{Channel}[c][c][.9]{\hspace*{-1mm}{Channel}}
\psfrag{fy}[c][c][1.1]{\hspace*{-1mm}$\substack{\vy_{\obsidx} = \mathbf{D} \mathbf{x}_{\obsidx} + \noisevec_{\obsidx} \\[2mm] \mathbf{Y} = (\vy_{1},\ldots,\vy_{\samplesize})}$}
\psfrag{Decoder}[c][c][.9]{\hspace*{-1mm}{Decoder}}
\psfrag{min}[c][c][1]{\hspace*{3mm}$\min\limits_{l \in [L]}\|\widehat{\mathbf{D}}(\mathbf{Y})\!-\!\mathbf{D}_{l} \|_{\rm{F}}$}
\psfrag{hatl}[c][c][.9]{\hspace*{-1mm}{$\hat{l}$}}
\psfrag{DeqDl}[c][c][1.2]{\hspace*{-1mm}$\mathbf{D} \!=\! \mathbf{D}_{l}$}
\psfrag{D0}[c][c][.9]{\hspace*{2mm}$\mathcal{D}_{0}$}
\psfrag{D2}[c][c][.9]{\hspace*{2mm}$\mathbf{D}^{(2)}$}
\psfrag{D3}[c][c][.9]{\hspace*{2mm}$\mathbf{D}^{(3)}$}
\psfrag{D4}[c][c][.9]{\hspace*{2mm}$\mathbf{D}^{(4)}$}
\psfrag{hatD}[c][c][.9]{\hspace*{4mm}$\widehat{\mathbf{D}}(\mathbf{Y})$}
\psfrag{x_m20}[c][c][.9]{\uput{0.3mm}[270]{0}{$-20$}}
\psfrag{x_m_10}[c][c][.9]{\uput{0.3mm}[270]{0}{$-10$}}
\centering
\hspace*{-0mm}\includegraphics[height=2.7cm]{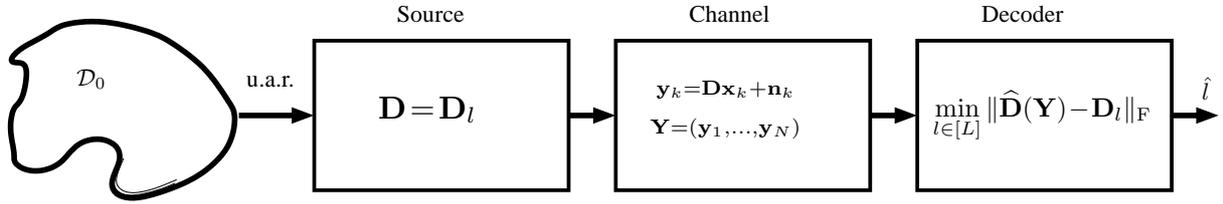}
\vspace*{-.5mm}
\renewcommand{\baselinestretch}{1.2}\small\normalsize
  \caption{Information-theoretic method for lower bounding the minimax risk.} 
\label{fig_comm_problem_inf_th_method}
\vspace*{3.5mm}
\end{figure}

Thus, within this approach, the estimation problem of DL is interpreted as a communication problem as illustrated in Fig.\ \ref{fig_comm_problem_inf_th_method}. The source selects the true dictionary $\mathbf{D}\!=\!\mathbf{D}_{\dicidx}$ by drawing u.a.r. an element $\mathbf{D}_{\dicidx}$ from the set $\mathcal{D}_{0}$. 
This element $\mathbf{D}_{\dicidx}$ then generates the ``channel output'' $\mY\!=\!\big(\vy_{1},\ldots,\vy_{\samplesize}\big)$ via the model \eqref{equ_linear_model} for $\samplesize$ channel uses. 
The observation model \eqref{equ_linear_model} acts as a channel model, relating the input $\mD\!=\!\mD_{\dicidx}$ to the 
output $\mY$. A crucial step in the information-theoretic approach is the analysis of the MI defined by \cite{coverthomas} 
\begin{equation} 
 I(\mY;\dicidx) \defeq \expect_{{\mY,\dicidx}}  \big\{ \log \frac{ p(\mY,\dicidx)}{p(\mY)p(\dicidx)} \big\}, \nonumber
\end{equation}
where $p(\mathbf{Y},\dicidx)$, $p(\mathbf{Y})$ and $p(\dicidx)$ denote the joint and marginal distributions, respectively, of the channel output $\mathbf{Y}$ and the random index $\dicidx$. As it turns out, a key challenge for applying this method to DL is that 
the model \eqref{equ_linear_model} does not correspond to a simple AWGN channel, for which the MI between output and input can be characterized easily. 
Indeed, the model \eqref{equ_linear_model} corresponds to a fading channel with the vector $\coeffvec$ representing fading coefficients. 
As is known from the analysis of non-coherent channel capacity, characterizing the MI between output and input for 
fading channels is much more involved than for AWGN channels \cite{LapidothMoser2003}. 
In particular, we require a tight upper bound on the MI $I(\mathbf{Y}; \dicidx)$ between the output $\mathbf{Y}$ and a random index $\dicidx$ which 
selects the input $\mathbf{D} = \mathbf{D}_{\dicidx}$ u.a.r. from a finite set $\mathcal{D}_{0} \subseteq {\mathcal{X}}(\mathbf{D}_{0},r)$. 
Upper bounding $I(\mathbf{Y};\dicidx)$ typically involves the analysis of the Kullback-Leibler (KL) divergence between 
the distributions of $\mathbf{Y}$ induced by different dictionaries $\mathbf{D}=\mathbf{D}_{\dicidx}$, $\dicidx \in [L]$. 

{Unfortunately}, an exact characterization of the KL divergence between Gaussian mixture models is in general not possible and one has to 
resort to approximations or bounds \cite{HersheyOlsen2007}. A main conceptual contribution of this work is a strategy to avoid evaluating 
KL divergences between Gaussian mixture models. Instead, similar to the approach of \cite{Wain2009TIT}, we assume that, in addition to the observation $\mathbf{Y}$, we also have access to some side information $\mathbf{T}(\coeffmtx)$, which depends only on the coefficient vector {$\coeffvec_{\obsidx}$}, for $\obsidx \in [\samplesize]$, stored column-wise in the matrix $\coeffmtx\!=\!\big(\coeffvec_{1},\ldots,\coeffvec_{\samplesize}\big)$. Clearly, any lower bound on the minimax risk for the situation with the additional side information $\mathbf{T}(\mathbf{X})$ is trivially also a lower bound for the case of no side information, since the optimal learning scheme for the latter situation may simply ignore the side information $\mathbf{T}(\mathbf{X})$. 
{As we will show rigorously in Appendix \ref{sec_appdx_tech}, we have the upper bound MI $I(\mathbf{Y};\dicidx) \leq \condmiyandlgivenx $, where $\condmiyandlgivenx$ is the conditional mutual information, given the side information $\mathbf{T}(\mathbf{X})$, between the observed data matrix $\mathbf{Y}$ and the random index $\dicidx$. Thus, in order to control the MI $I(\mathbf{Y};\dicidx)$ it is sufficient 
to control the conditional MI $\condmiyandlgivenx$, which turns out to be a much easier task.}
We will use two specific choices for $\mathbf{T}(\mathbf{X})$: $\mathbf{T}(\mathbf{X}) \!=\! \mathbf{X}$ and $\mathbf{T}(\mathbf{X}) \!=\! \supp(\mathbf{X})$. 
{The choice $\mathbf{T}(\mathbf{X}) \!=\! \mathbf{X}$ will yield tighter bounds for the case of high SNR, while the choice $\mathbf{T}(\mathbf{X}) \!=\! \supp(\mathbf{X})$ yields 
more accurate bounds in the low SNR regime.} 
As detailed in Section \ref{proof_architecture_main_result}, the problem of upper bounding $\condmiyandlgivenx$ becomes tractable for both choices.


\section{Lower Bounds on the Minimax Risk For DL} 
\label{sec_main_result}

We now state our main results, i.e., lower bounds on the minimax risk of DL. 
The first bound applies to any distribution of the coefficient vector $\coeffvec$, requiring only the existence of the covariance matrix $\mathbf{\Sigma}_{x}$. 
{Two further, more specialized, lower bounds apply to sparse coefficient vectors and moreover require the underlying dictionary $\mD$ in \eqref{equ_single_linear_model} to satisfy a restricted isometry property (RIP) \cite{RauhutFoucartCS}}.

\subsection{General Coefficients}
\label{sub_suc_general_coeff}

In this section, we consider the DL problem based on the model \eqref{equ_linear_model} with a zero-mean random coefficient vector $\mathbf{x}$. We make no further assumptions on the statistics of $\mathbf{x}$ except that the covariance matrix $\mathbf{\Sigma}_{x}$ exists. For this setup, the side information $\mathbf{T}(\mathbf{X})$ for the derivation of lower bounds on the minimax risk will be chosen as the coefficients itself, i.e., $\mathbf{T}(\mathbf{X}) \!=\! \mathbf{X}$. 
Our {first} main result is the following lower bound on the minimax risk for the DL problem.
\begin{theorem}
\label{thm_main_result}
Consider a DL problem based on $\samplesize$ i.i.d. observations following the model \eqref{equ_linear_model} and with true dictionary satisfying \eqref{eq_dic_belongs_to_small_ball} 
for some {$r \leq 2\sqrt{\coefflen}$}. Then, if 
\vspace*{-2mm}
\begin{equation} 
\label{equ_cond_coefflen_measlen}
\coefflen(\measlen\!-\!1) \geq {50}, 
\vspace*{-2mm}
\end{equation}
the minimax risk $\minimaxrisk$ is lower bounded as
\begin{equation}
\label{equ_conditions_s_2_first_charac_minimax_risk}
\minimaxrisk \geq (1/{320}) \min \bigg\{ r^{2},   \frac{\sigma^2}{\samplesize \| \mathbf{\Sigma}_{x}\|_{2}}  (\coefflen(\measlen\!-\!1)/10 -1) \bigg\}.
\end{equation}
\end{theorem}

The first bound in \eqref{equ_conditions_s_2_first_charac_minimax_risk}, i.e., $\minimaxrisk \geq r^{2}/{320}$,\footnote{{The constant $1/40$ is an artifact of our proof technique and might be improved by a more pedantic analysis.}} complies {(up to fixed constants)} with the worst case MSE of a dumb estimator $\widehat{\mathbf{D}}$ which ignores 
the observation $\mY$ and always delivers a fixed dictionary $\mathbf{D}_{1} \in {\mathcal{X}}(\mathbf{D}_{0},r)$. Since the true dictionary $\mD$ also belongs to the neighborhood ${\mathcal{X}}(\mD_{0},r)$, the {MSE} of this estimator is upper bounded by 
\begin{equation}
\| \widehat { \mD } \!-\! \mathbf{D} \|_{\rm{F}}^{2}  =  \| \mD_{1} \!-\! \mathbf{D} \|_{\rm{F}}^{2}  =\big( \| \mD_{1} \!-\! \mD_{0} \|_{\rm{F}}\!+\! \| \mD_{0} \!-\! \mD \|_{\rm{F}} \big)^{2} \stackrel{\eqref{eq_dic_belongs_to_small_ball}}{\leq} 4 r^2.\nonumber
\end{equation} 
The second bound in \eqref{equ_conditions_s_2_first_charac_minimax_risk} (ignoring constants) is essentially the minimax risk $\varepsilon'$ of a {simple} signal in noise problem 
\begin{equation} 
\label{equ_signal_in_noise_model}
\mathbf{z} = \mathbf{s} + \noisevec 
\end{equation}
with AWGN $\noisevec \sim \mathcal{N}(\mathbf{0}, \frac{\sigma^{2}}{\| \mathbf{\Sigma}_{x} \|_{2} } \mathbf{I}_{{\coefflen(\measlen\!-\!1)}})$ and the unknown non-random signal $\mathbf{s}$ of dimension $\coefflen(\measlen\!-\!1)$, 
which is also the dimension of the oblique manifold $\mathcal{D}$ \cite{Absil2006}. 
A standard result in classical estimation theory is that, given {the observation of} $\samplesize$ i.i.d.\ realizations $\mathbf{z}_{\obsidx}$ of the vector $\mathbf{z}$ in \eqref{equ_signal_in_noise_model}, the minimax risk $\varepsilon'$ of 
estimating $\mathbf{s} \in \mathbb{R}^{\coefflen(\measlen\!-\!1)}$ is \cite[Exercise 5.8 on pp. 403]{LC}
\begin{equation}
\varepsilon' =\frac{\sigma^{2}}{\samplesize \| \mathbf{\Sigma}_{x} \|_{2} }  \coefflen(\measlen\!-\!1).  
\end{equation} 

For fixed {ratio} $ \| \mathbf{\Sigma}_{x}\|_{2} /\sigma^{2}$, the bound \eqref{equ_conditions_s_2_first_charac_minimax_risk} predicts that $\samplesize \!=\! \Theta(\coefflen \measlen)$ samples are required for accurate DL. 
{Remarkably,} this scaling matches the scaling of the sample size found in \cite{VainManBruck11} to be sufficient for successful DL. 
Note, however, that the analysis \cite{VainManBruck11} is based on the sparse representation 
error of a dictionary, whereas we target the Frobenius norm of the deviation from the true underlying dictionary. 

\subsection{Sparse Coefficients} 
\label{sub_sec_sparse_coding}

In this section {we focus on a particular subclass of probability distributions for the zero mean coefficient vector} $\vx$ in \eqref{equ_linear_model}. More specifically, {the random support $\supp(\vx)$
of the coefficient vector $\vx$ is assumed to be distributed uniformly over the set $\Xi \defeq \{ \mathcal{S} \subseteq [\coefflen]: | \mathcal{S}| = \sparsity \}$, i.e., 
\begin{equation} 
\label{equ_sparse_coeff_support_prob}
\prob( \supp(\vx) = \mathcal{S}) = \frac{1}{|\Xi|} = \frac{1}{\binom{\coefflen}{\sparsity}} \mbox{, for any } \mathcal{S} \in \Xi.
\end{equation}
We also assume that, conditioned on the support $\mathcal{S}=\supp(\vx)$, the non-zero entries of $\vx$ are i.i.d. with variance $\sigma_{a}^{2}$, i.e., in particular 
\vspace*{-2mm}
\begin{equation} 
\label{equ_sparse_coeff_cond_cov_matrix}
\expect_{\vx} \{ \vx_{\mathcal{S}} \vx_{\mathcal{S}}^{T} | \mathcal{S} \} = \sigma_{a}^{2} \mathbf{I}_{\sparsity}. 
\vspace*{-6mm}
\end{equation} 
}

The sparse coefficient support model \eqref{equ_sparse_coeff_support_prob} is useful for performing sparse coding of the observed samples $\vy_{\obsidx}$. 
Indeed, once we have learned the dictionary $\mathbf{D}$, we can estimate for each observed sample $\vy_{\obsidx}$, using a standard CS recovery method, the sparse coefficient vector $\mathbf{x}_{\obsidx}$. 
Sparse source coding is then accomplished by using the sparse coefficient vector to represent the signal $\vy_{\obsidx}$. 
{For sparse source coding to be robust against noise, one has to require the underlying dictionary $\mD$ to be well conditioned for sparse signals. 
While there are various ways of quantifying the conditioning of a dictionary, e.g., based on the dictionary coherence \cite{GribonvalSchnass2010,aganjaneta13}, 
we will focus here on the restricted isometry property (RIP) \cite{RauhutFoucartCS,CandesRIPImplCS08,CandesDavenport2013}.
A dictionary $\mD$ is said to satisfy the RIP of order $\sparsity$ with constant $\delta_{\sparsity}$ if 
\begin{equation} 
\label{equ_def_RIP}
(1\!-\!\delta_{\sparsity})  \| \mathbf{z} \|^{2} \leq \| \mathbf{D} \mathbf{z} \|^{2} \leq (1\!+\!\delta_{\sparsity}) \| \mathbf{z} \|^{2} \mbox{, for any } \mathbf{z} \in \mathbb{R}^{\coefflen} \mbox{ such that } \| \mathbf{z} \|_{0} \leq  \sparsity.
\vspace*{-6mm}
\end{equation}  
}

{Let us formally define the signal-to-noise ratio (SNR) for the observation model \eqref{equ_linear_model} as  
\begin{equation}
\label{equ_def_SNR}
{\rm SNR} \defeq \expect_{\vx} \{ \| \mD \vx \|^{2}_{2} \} / \expect_{\noisevec} \{ \| \noisevec \|^{2}_{2} \}. 
\end{equation} 
Note that the $\rm{SNR}$ depends on the unknown underlying dictionary $\mD$. However, if $\mD$ satisfies the RIP \eqref{equ_def_RIP} with constant $\delta_{\sparsity}$, then we obtain the characterization  
\begin{equation}
\label{equ_doulbe_inequ_SNR}
 \frac{ (1-\delta_{\sparsity} )\sparsity \sigma_{a}^{2}}{ \measlen \sigma^{2}} \leq {\rm SNR} \leq  \frac{ (1+\delta_{\sparsity} )\sparsity \sigma_{a}^{2}}{ \measlen \sigma^{2}}
\end{equation}
which depends on $\mD$ only via the RIP constant $\delta_{\sparsity}$. For a small constant $\delta_{\sparsity}$, \eqref{equ_doulbe_inequ_SNR} justifies the approximation ${\rm SNR} \approx \frac{\sparsity \sigma_{a}^{2}}{\measlen \sigma^2}$. 
}

{As can be verified easily, any random coefficient vector $\vx$ conforming with \eqref{equ_sparse_coeff_support_prob} and \eqref{equ_sparse_coeff_cond_cov_matrix} 
possesses a finite covariance matrix, given explicitly by 
\begin{equation}
\label{equ_cov_matrix_sparse_coeffs}
\mathbf{\Sigma}_{x} = (\sparsity/\coefflen)  \sigma_{a}^{2} \mathbf{I}_{\coefflen}. 
\end{equation} 
Therefore we can invoke Theorem \ref{thm_main_result}, which, combined with \eqref{equ_cov_matrix_sparse_coeffs} and  \eqref{equ_doulbe_inequ_SNR}, yields the following corollary.}
\begin{corollary}
\label{cor_main_result_sparse_vectors}
Consider a DL problem based on $\samplesize$ i.i.d. observations according to the model \eqref{equ_linear_model} and with true dictionary satisfying \eqref{eq_dic_belongs_to_small_ball} for some {$r \leq 2\sqrt{\coefflen}$}. 
Furthermore, the random coefficient vector $\coeffvec$ in \eqref{equ_linear_model} {conforms with \eqref{equ_sparse_coeff_support_prob} and \eqref{equ_sparse_coeff_cond_cov_matrix}}. {If the dictionary $\mD$ satisfies the RIP \eqref{equ_def_RIP} with RIP-constant $\delta_{\sparsity}\!\leq\!1/2$ and moreover}
\vspace*{-2mm}
\begin{equation} 
\label{equ_cond_coefflen_measlen_sparse}
\coefflen(\measlen\!-\!1) \geq {50}, 
\vspace*{-2mm}
\end{equation}
{then} the minimax risk $\minimaxrisk$ is lower bounded as
\begin{equation}
\label{equ_conditions_s_2_first_charac_minimax_risk_cor_sparse}
\minimaxrisk \geq (1/{320}) \min \bigg\{ r^{2},   {\frac{2\coefflen}{{\rm SNR}\samplesize \measlen}  (\coefflen(\measlen\!-\!1)/10\!-\!1)} \bigg\}.
\end{equation}
\end{corollary}

{For sufficiently large sample size $\samplesize$ the second bound in \eqref{equ_conditions_s_2_first_charac_minimax_risk_cor_sparse} will be in force, and we obtain 
a scaling of the minimax risk as $\minimaxrisk = \Theta(\coefflen^{2}/(\samplesize {\rm SNR}))$. In particular, this bound suggests a decay of the worst case MSE via $1/\samplesize$.} This agrees with empirical results in \cite{Jenatton2012}, indicating that the MSE of popular DL methods typically decay with $1/\samplesize$. Moreover, the dependence on the sample size via $1/\samplesize$ is theoretically sound, since averaging the outcomes of a learning scheme over $\samplesize$ independent observations reduces the estimator variance by $1/\samplesize$. Note that, as long as the first bound in \eqref{equ_conditions_s_2_first_charac_minimax_risk_cor_sparse} is not in force, the overall lower bound \eqref{equ_conditions_s_2_first_charac_minimax_risk_cor_sparse} scales with {$1/{\rm SNR}$}, which agrees with the basic behavior of the upper bound derived in \cite{Jenatton2012} on the distance of the closest local minimum of \eqref{equ_minimization_ell_1_dic_learning} to the true dictionary $\mathbf{D}$.

If we consider a fixed {SNR (cf.\ \eqref{equ_def_SNR})}, our lower bound predicts that for a vanishing minimax risk $\minimaxrisk$ the sample size $\samplesize$ has to scale as $\samplesize \!=\!\Theta(\coefflen^{2})$. 
This scaling is considerably smaller than the sample size requirement $\samplesize\!=\!\mathcal{O}(\coefflen^3 \measlen)$, which \cite{Jenatton2012} proved to be sufficient in the noisy and over-complete setting, such that minimizing \eqref{equ_minimization_ell_1_dic_learning} yields an accurate estimate of the true dictionary $\mD$. However, for vanishing sparsity rate ($\sparsity/\coefflen \rightarrow 0$), the scaling $\samplesize \!=\!\Theta(\coefflen^{2})$ matches the required sample size of the algorithms put forward in \cite{aganjaneta13,argemo13}, certifying that, {for extremely sparse signals}, they perform close to the information-theoretic optimum for fixed SNR.  


We will now derive an alternative lower bound on the minimax risk for DL {based on the sparse coefficient model \eqref{equ_sparse_coeff_support_prob} and \eqref{equ_sparse_coeff_cond_cov_matrix} by additionally assuming the non-zero coefficients 
to be Gaussian.  
In particular,} let us denote by $\mathbf{P}$ a random matrix 
which is drawn u.a.r. from the set of all permutation matrices of size $\coefflen \times \coefflen$. Furthermore, we denote by $\mathbf{z} \in \mathbb{R}^{\sparsity}$ a multivariate 
normal random vector with zero mean and covariance matrix $\mathbf{\Sigma}_{z} = \sigma_{a}^{2} \mathbf{I}_{\sparsity}$. Based on the matrix $\mathbf{P}$ and vector $\mathbf{z}$, we generate 
the coefficient vector $\vx$ as
\begin{equation} 
\label{equ_model_sparse_coeffs}
\coeffvec = {\mathbf{P}  \big(\vz^{T},\mathbf{0}_{1\!\times\!(\coefflen\!-\!\sparsity)}\big)^{T}} \mbox{ with } \vz \sim \mathcal{N}(\mathbf{0}, \sigma_{a}^{2} \mathbf{I}_{\sparsity}).
\end{equation} 
{Theorem \ref{thm_main_result_sparse_coeff} below presents} a lower bound on the minimax risk for the low SNR 
regime where {${\rm SNR} \leq (1/(9 \sqrt{80}))  \measlen/(2\sparsity)$}. 

\begin{theorem}
\label{thm_main_result_sparse_coeff}
Consider a DL problem based on the model \eqref{equ_linear_model} such that \eqref{eq_dic_belongs_to_small_ball} holds with some {$r \leq 2\sqrt{\coefflen}$ and the underlying dictionary $\mD$ satisfies 
the RIP of order $\sparsity$ with constant $\delta_{\sparsity} \leq 1/2$ (cf. \eqref{equ_def_RIP}).}
We assume the coefficients $\mathbf{x}$ in \eqref{equ_linear_model} to be distributed according to \eqref{equ_model_sparse_coeffs} with {${\rm SNR} \leq (1/(9 \sqrt{80}))  \measlen/(2\sparsity)$}.
Then, 
if
\vspace*{-2mm}
\begin{equation} 
\label{equ_cond_coefflen_measlen}
\coefflen (\measlen\!-\!1) \geq {50}, 
\vspace*{-2mm}
\end{equation}
the minimax risk $\minimaxrisk$ is lower bounded as
\begin{equation}
\label{equ_conditions_s_2_first_charac_minimax_risk_sparse}
\minimaxrisk \geq (1/{12960}) \min \bigg\{ r^{2}{/\sparsity},  \frac{ \coefflen}{{\rm SNR}^{2}\samplesize \measlen^{2}}   (\coefflen(\measlen\!-\!1)/10 -1) \bigg\}.
\end{equation}
\end{theorem}

The main difference between the bounds \eqref{equ_conditions_s_2_first_charac_minimax_risk_cor_sparse} and \eqref{equ_conditions_s_2_first_charac_minimax_risk_sparse} is their dependence on the {SNR \eqref{equ_def_SNR}}. While the bound \eqref{equ_conditions_s_2_first_charac_minimax_risk_cor_sparse}, which applies to arbitrary coefficient statistics and does not exploit the sparse structure of the model \eqref{equ_model_sparse_coeffs}, depends on the SNR via $1/{\rm SNR}$, the 
bound \eqref{equ_conditions_s_2_first_charac_minimax_risk_sparse} shows a dependence via $1/{\rm SNR}^{2}$. Thus, in the low SNR regime where ${\rm SNR} \ll 1$, the bound \eqref{equ_conditions_s_2_first_charac_minimax_risk_sparse} tends to be tighter, i.e. higher, than the bound \eqref{equ_conditions_s_2_first_charac_minimax_risk_cor_sparse}. 

We now show that the dependence of the bound \eqref{equ_conditions_s_2_first_charac_minimax_risk_sparse} on the SNR via $1/{\rm SNR}^2$ agrees with the basic behavior of the constrained \CRBfull (CCRB) \cite{StoicaNgCCRB}. Indeed, if we assume for simplicity that $\coefflen=\sparsity=1$ and the true dictionary (which is now a vector) is $\mathbf{d} = \mathbf{e}_{1}$, we obtain for the CCRB \cite[Thm. 1]{StoicaNgCCRB}
\begin{equation}
\label{equ_eq_CCRB_DicLearn}
\expect_{{\mY}}  \{ (\widehat{\vd}{(\mY)} - \vd) (\widehat{\vd}{(\mY)} - \vd)^{T} \}  \succeq  {\frac{1}{{\rm SNR}^{2}\measlen^2\samplesize}} ( \mathbf{I} - \mathbf{e}_{1} \mathbf{e}_{1}^{T})
\end{equation} 
for any unbiased learning scheme $\widehat{\vd}{(\mY)}$, i.e., which satisfies $\expect_{{\mY}}  \{ \widehat{\vd}{(\mY)} \} = \vd$.\footnote{Using the notation of \cite{StoicaNgCCRB}, we obtained \eqref{equ_eq_CCRB_DicLearn} from 
\cite[Thm. 1]{StoicaNgCCRB} by using the matrix $\mathbf{U}=(\mathbf{e}_{2},\ldots,\mathbf{e}_{\measlen})$ which forms an orthonormal basis for {the} null space of the gradient mapping $\mathbf{F}(\vd) = \frac{ \partial f(\vd)) }{\partial \vd}$ with the constraint function $f(\vd) = \| \vd\|_{2}^{2}\!-\!1$. Moreover, for evaluating \cite[Thm. 1]{StoicaNgCCRB} we used the formula $J_{k,l} = (1/2) \trace\big\{\mathbf{C}^{-1}(\vd)\frac{\partial \mathbf{C}(\vd)}{\partial d_{k}}\mathbf{C}^{-1}(\vd)\frac{\partial \mathbf{C}(\vd)}{\partial d_{l}} \big\}$ \cite{Mardia1984} for the elements of the Fisher information matrix, which applies for a Gaussian observation with zero mean and whose covariance matrix $\mathbf{C}(\vd)$ depends on the parameter vector $\vd$.} Thus, in this simplified setting, the dependence of the minimax bound \eqref{equ_conditions_s_2_first_charac_minimax_risk} on the SNR via $1/{\rm SNR}^{2}$ is also reflected by the CCRB. 

{Let us finally highlight that the bound in Theorem \ref{thm_main_result_sparse_coeff} is derived by exploiting the (conditional) Gaussianity of the non-zero entries in the coefficient vector. By contrast, the bounds in Theorem \ref{thm_main_result} and Corollary \ref{cor_main_result_sparse_vectors} do not require the non-zero entries to be Gaussian.}

{
\subsection{A partial converse}} 
{
Given the lower bounds on the minimax risk presented in Sections \ref{sub_suc_general_coeff} and \ref{sub_sec_sparse_coding} it is natural to ask wether these are sharp, i.e., 
there exist DL schemes whose worst case MSE comes close to the lower bounds. To this end, we consider a simple instance of the DL problem and analyze the MSE of a very basic 
DL scheme. As it turns out, in certain regimes, the worst case MSE of this simple DL approach essentially matches the lower bound \eqref{equ_conditions_s_2_first_charac_minimax_risk_cor_sparse}.
}

{
\begin{theorem}
\label{thm_partial_converse}
Consider a DL problem based on $\samplesize$ i.i.d. observations according to the model \eqref{equ_linear_model} and with true dictionary satisfying \eqref{eq_dic_belongs_to_small_ball} with $\mD_{0}=\mathbf{I}$ and some {$r \leq 2\sqrt{\coefflen}$}. Furthermore, the random coefficient vector $\coeffvec$ in \eqref{equ_linear_model} {conforms with \eqref{equ_sparse_coeff_support_prob} and \eqref{equ_sparse_coeff_cond_cov_matrix}}. Moreover, the non-zero entries of $\coeffvec$ have magnitude equal to one, i.e., $\coeffvec \in \{-1,0,1\}^{\coefflen}$. If $r\sqrt{\sparsity}\leq 1/10$ and $\sigma \leq 0.4$, there exists a DL scheme whose MSE satisfies 
\begin{equation}
\label{equ_upper_bound_MSE}
\expect_{{\mY}} \{ \| \widehat{\mD}{(\mY)} - \mD \|^{2}_{\rm F} \}  \leq  4(\coefflen^2/\samplesize)  \big[ (1-r)^{2}/{\rm SNR}+1\big]+ 2 \coefflen \exp(- \coefflen \samplesize 0.4^2/(2\sigma^2)),\end{equation}
for any $\mD \in \mathcal{X}(\mD_{0},r)$.
\end{theorem}
}

{The proof of Theorem \ref{thm_partial_converse}, to be found at the end of Section \ref{proof_architecture_main_result}, will be based on a straightforward analysis of a simple DL method which is given by the following algorithm. 
{
\begin{algorithm}
\label{alg_simple}
\vspace*{2mm}
Input: data matrix $\mY=(\vy_{1},\ldots,\vy_{\samplesize})$ \\
Output: learned dictionary $\widehat{\mD}(\mY)$ \\
\begin{enumerate} 
\item Compute an estimate $\widehat{\mX}$ of the coefficient matrix $\mX=(\vx_{1},\ldots,\vx_{\samplesize})$ by simple element-wise thresholding, i.e., 
\begin{equation} 
\label{equ_def_coeff_decoder}
\widehat{\mX} = \big(\widehat{\vx}_{1},\ldots,\widehat{\vx}_{\samplesize}\big) \mbox{, with } \hat{x}_{\obsidx,l} = \begin{cases} 1 & \mbox{ , if } y_{\obsidx,l} > 0.5 \\ 0 & \mbox { , if } |y_{\obsidx,l}| \leq 0.5 \\  -1 & \mbox{ , if } y_{\obsidx,l} < 0.5 \end{cases}
\end{equation} 
\item For each column-index $j \in [\coefflen]$, define 
\begin{equation}
\widetilde{\vd}_{j} \defeq \frac{\coefflen}{\samplesize\sparsity} \sum_{\obsidx \in [\samplesize]} \hat{x}_{\obsidx,j} \vy_{\obsidx}.
\end{equation} 
\item Output 
\begin{equation} 
\widehat{\mD}(\mY) \defeq \big( \widehat{\vd}_{1},\ldots,\widehat{\vd}_{\coefflen} \big) \mbox{, with } \widehat{\vd}_{l} = \mathbf{P}_{\overline{\mathcal{B}}(\mathbf{e}_{l},\rho)} \widetilde{\vd}_{l}.
\end{equation} 
Here, $\mathbf{P}_{\overline{\mathcal{B}}(1)} \vd \defeq \argmin_{ \vd' \in \overline{\mathcal{B}}(1)} \| \vd' - \vd\|_{2}$ denotes the projection of 
the vector $\vd\!\in\!\mathbb{R}^{\measlen}$ on the closed unit ball $\overline{\mathcal{B}}(1)\defeq \{ \vd' \in \mathbb{R}^{\measlen}: \| \vd' \|_{2} \leq 1 \}$.   
\end{enumerate}
\end{algorithm}
}
{Note that the learned dictionary $\widehat{\mD}(\mY)$ obtained by Algorithm \ref{alg_simple} might not have unit-norm columns so that it 
might not belong to the oblique manifold $\mathcal{D}$. While this is somewhat counter-intuitive, as the true dictionary $\mD$ belongs to $\mathcal{D}$, this fact is 
not relevant for the derivation of upper bounds on the MSE incurred by $\widehat{\mD}(\mY)$.}
}

{According to Theorem \ref{thm_partial_converse}, in the low-SNR regime, i.e., where ${\rm SNR} = o(1)$, and for sufficiently small noise variance, such that $\sigma\leq0.4$ and 
\begin{equation}
 \coefflen \exp(- \coefflen \samplesize 0.4^2/(2\sigma^2)) = o( (\coefflen^2/\samplesize) (1-r)^{2}/{\rm SNR}),
\end{equation} 
the MSE of the DL scheme given by Algorithm \ref{alg_simple} scales as  
\begin{equation}
\label{equ_upper_bound_MSE_simple_DL}
\expect_{{\mY}} \{ \| \widehat{\mD}(\mY) \!-\! \mD \|^{2}_{\rm F} \}  = \mathcal{O}\bigg(\frac{\coefflen^2 (1-r)^{2}}{\samplesize {\rm SNR}}\bigg).
\end{equation}
We highlight that the scaling of the upper bound \eqref{equ_upper_bound_MSE_simple_DL} essentially matches the scaling of the lower bound \eqref{equ_conditions_s_2_first_charac_minimax_risk_cor_sparse}, certifying that the bound of Corollary \ref{cor_main_result_sparse_vectors} is tight in certain regimes.}

\section{Proof of the main results}
\label{proof_architecture_main_result}

Before stating the detailed proofs of Theorem \ref{thm_main_result} and Theorem \ref{thm_main_result_sparse_coeff}, we present the key idea behind and the main ingredients 
used for their proofs. At their core, the proofs of Theorem \ref{thm_main_result} and Theorem \ref{thm_main_result_sparse_coeff} are based on the construction of a finite set $\mathcal{D}_{0} \triangleq \{\mathbf{D}_{1}, \ldots,\mathbf{D}_{L} \}  \subseteq \mathcal{D}$ (cf. \eqref{equ_def_dictionary_set_normalized_cols}) of $L$ distinct dictionaries having the following desiderata: 
\begin{itemize} 
\item For any two dictionaries $\mathbf{D}_{\dicidx}, \mathbf{D}_{\dicidx'} \in \mathcal{D}_{0}$,
\begin{equation} 
\label{equ_cond_frob_norm_diff_delta}
\| \mathbf{D}_{\dicidx} - \mathbf{D}_{\dicidx'} \|^{2}_{\text{F}} \!\geq\! \bar{\delta}_{\dicidx,\dicidx'} 8 \varepsilon.
\end{equation} 
\item If the true dictionary in \eqref{equ_linear_model} is chosen as $\mathbf{D} = \mathbf{D}_{\dicidx} \in \mathcal{D}_{0}$, where $\dicidx$ is selected u.a.r. from $[L]$, then the conditional MI 
between $\mathbf{Y}$ and $\dicidx$, given the side information $\mathbf{T}(\mathbf{X})$,\footnote{{Particular choices for $\mathbf{T}(\mX)$ are discussed at the end of Section \ref{sec_inf_th_lower_bounds_minimax_risk}.}} is bounded as 
\begin{align}
\label{equ_desirata_av_mutual_info_y_index}
\condmiyandlgivenx \leq \eta 
\end{align}
with some small $\eta$.
\end{itemize}

For the verification of the existence of such a set $\mathcal{D}_{0}$, we rely on the following result:
\begin{lemma}
\label{lem_existence_set_B_C}
For $P \in \mathbb{N}$ such that 
\begin{equation} 
\label{equ_cond_set_size_hoeffding_bound}
\log(P)/d < (1\!-\!2{/10})^2/4, 
\end{equation} 
there exists a set $\mathcal{P} \defeq \{\mathbf{b}_{\dicidx} \}_{\dicidx \in [P]}$ of $P$ distinct binary vectors $\mathbf{b}_{\dicidx}  \in \{-1,1\}^{d}$ satisfying 
\begin{equation}
\label{equ_conditiona_diff_at_least_n_half_somewhat_b}
\| \mathbf{b}_{\dicidx} - \mathbf{b}_{\dicidx'} \|_{0} \geq  {d/10}, \mbox{ for any two different indices } \dicidx,\dicidx' \in [P].
\end{equation}
\end{lemma}
 
\begin{proof}
We construct the set $\mathcal{P}$ sequentially by drawing i.i.d. realizations $\mathbf{b}_{\dicidx}$ from a standard Bernoulli vector $\mathbf{b} \in \{-1,1\}^{d}$. Consider 
two different indices $\dicidx, \dicidx' \in [P]$. Define the vector $\tilde{\mathbf{b}} \defeq \mathbf{b}_{\dicidx} \odot \mathbf{b}_{\dicidx'}$ by element-wise multiplication and observe that  
\begin{equation}
\label{equ_relation_ell_0_b_j_b_j_prime_sum_tilde_b}
\| \mathbf{b}_{\dicidx} - \mathbf{b}_{\dicidx'} \|_{0} = (1/2) \bigg(\coefflen - \sum_{r \in [d]} \tilde{b}_r \bigg). 
\end{equation}  
Each one of the three vectors $\mathbf{b}_{\dicidx}, \mathbf{b}_{\dicidx'}, \widetilde{\mathbf{b}} \in \{-1,1\}^{d}$ contains zero-mean i.i.d. Bernoulli variables. 
We have 
\begin{align}
\label{equ_bound_ell_0_norm_sum_b_tilde}
\prob \{ \| \mathbf{b}_{\dicidx} - \mathbf{b}_{\dicidx'} \|_{0} \leq  d/10 \} &\stackrel{\eqref{equ_relation_ell_0_b_j_b_j_prime_sum_tilde_b}}{=} \prob \{ (d - \sum_{r \in [d]} \tilde{b}_{r})/2\leq d/10 \} \nonumber \\ 
& = \prob \big\{ \sum_{r \in [d]} \tilde{b}_{r}  \geq  d (1\!-\!2/10) \big\}.
\end{align} 
According to Lemma \ref{lem_hoeffing}, 
\begin{equation} 
\label{equ_hoeffding_bound_tilde_b}
\prob \{  \sum_{r \in [d]} \tilde{b}_{r} \geq  (1\!-\!2/10) d \} \leq  \exp( -  d  (1\!-\!2/10)^2 /2).  
\end{equation} 
Taking a union bound over all $\binom{P}{2}$ pairs $\dicidx,\dicidx' \in [P]$, we have from \eqref{equ_bound_ell_0_norm_sum_b_tilde} and \eqref{equ_hoeffding_bound_tilde_b} that the probability of $P$ i.i.d. draws $\{ \mathbf{b}_{\dicidx} \}_{\dicidx \in [P]}$ 
violating \eqref{equ_conditiona_diff_at_least_n_half_somewhat_b} is upper bounded by 
\begin{equation}
P_{1} \leq  \exp( -   d  (1\!-\!2/10)^2 /2 + 2 \log P),
\end{equation} 
which is strictly lower than $1$ if \eqref{equ_cond_set_size_hoeffding_bound} is valid. Thus, there must exist at least one set $\mathcal{P}=\{\mathbf{b}_{\dicidx}\}_{\dicidx \in [P]}$ of cardinality $P$ whose elements satisfy \eqref{equ_conditiona_diff_at_least_n_half_somewhat_b}. 
\end{proof} 

The following result gives a sufficient condition on the cardinality $L$ and threshold $\eta$ such that there exists {at} least one subset $\mathcal{D}_{0} \subseteq \mathcal{D}$ of $L$ distinct dictionaries satisfying \eqref{equ_cond_frob_norm_diff_delta} and \eqref{equ_desirata_av_mutual_info_y_index}.

\begin{lemma}
\label{lem_existence_dicationary_set_desiredata}
Consider a DL problem based on the {generative} model \eqref{equ_linear_model} such that \eqref{eq_dic_belongs_to_small_ball} holds with some {$r \leq 2\sqrt{\coefflen}$}. 
If $(\measlen \!-\! 1)\coefflen \geq {50}$, 
there exists a set $\mathcal{D}_{0} \subseteq \mathcal{D}$ of cardinality $L \!=\! 2^{(\measlen\!-\!1)\coefflen/5}$ such that 
 \eqref{equ_cond_frob_norm_diff_delta} and \eqref{equ_desirata_av_mutual_info_y_index} (for the side information $\mathbf{T}(\mX)\!=\!\mX$) are satisfied with
\begin{equation}
\label{equ_bound_Q_D_0}
\eta =  {320} \samplesize \| \mathbf{\Sigma}_{x} \|_{2}  \varepsilon/\sigma^{2}
\end{equation}
and
\begin{equation} 
\label{equ_cond_epsilon_exist_dic_set}
\varepsilon \leq r^{2}/{320}.
\end{equation} 
\end{lemma}

\begin{proof}
According to Lemma \ref{lem_existence_set_B_C}, for $(\measlen\!-\!1) \coefflen \!\geq\! {50}$, 
there is a set of $L$ matrices $\mathbf{D}_{1,l} \!\in\! (1/\sqrt{{4}(\measlen\!-\!1)\coefflen}) \{-1,1\}^{(\measlen\!-\!1) \times \coefflen}$, $l \in [L]$ with $L\geq 2^{(\measlen\!-\!1)\coefflen/5}$, such that 
\begin{equation} 
\label{equ_min_distance_D_1}
\| \mD_{1,\dicidx}\!-\! \mD_{1,\dicidx'} \|_{\rm{F}}^{2} \geq {1/40} \mbox{ for } \dicidx \neq \dicidx'.
\end{equation} 
Since the matrices $\mD_{1,\dicidx} \!\in\! \mathbb{R}^{{(\measlen\!-\!1)} \times \coefflen}${, for $\dicidx \in [L]$,} have entries with values in $(1/\sqrt{{4}(\measlen\!-\!1)\coefflen}) \{-1,1\}$ their columns all have norm equal to $1/\sqrt{{4}\coefflen}$. 

Based on the matrices $\mathbf{D}_{1,\dicidx}{\in \mathbb{R}^{(\measlen\!-\!1) \times \coefflen}}$, we now construct a modified set of matrices {$\mathbf{D}_{2,\dicidx} \!\in\! \mathbb{R}^{\measlen\times\coefflen}$}, $\dicidx \in [L]$. Let $\mathbf{U}_{j}$ denote {an arbitrary $\measlen \times \measlen$} unitary matrix satisfying 
\begin{equation} 
\label{equ_orthog_d_0_e_1}
\mathbf{d}_{0,j} = \mathbf{U}_{j}  \mathbf{e}_{1}.
\end{equation} 
Here, $\mathbf{d}_{0,j}$ denotes the $j$th column of $\mathbf{D}_{0}{\in \mathbb{R}^{\measlen \!\times\! \coefflen}}$. 
Then, we define the matrix $\mathbf{D}_{2,\dicidx}$ column-wise, by constructing its $j$th column $\mathbf{d}_{2,\dicidx,j}$ as 
\begin{equation}
\label{equ_def_d_2_d_1_colwise}
\mathbf{d}_{2,\dicidx,j} = \mathbf{U}_{j} \begin{pmatrix} 0 \\ \mathbf{d}_{1,\dicidx,j} \end{pmatrix},
\end{equation}
where $\mathbf{d}_{1,\dicidx,j}$ is the $j$th column of the matrix $\mathbf{D}_{1,\dicidx}$.
Note that, for any $\dicidx \in [L]$, the $j$th column $\mathbf{d}_{2,\dicidx,j}$ of $\mathbf{D}_{2,\dicidx}$ is orthogonal to the column $\mathbf{d}_{0,j}$ and has norm equal to $1/\sqrt{{4}\coefflen}$, i.e., 
\begin{equation} 
\label{equ_col_D_2_orthog_norm1_p}
{\rm{diag}} \{ \mathbf{D}_{0}^{T} \mathbf{D}_{2,\dicidx} \} = \mathbf{0} \mbox{, and } {\rm{diag}} \{  \mathbf{D}_{2,\dicidx}^{T}  \mathbf{D}_{2,\dicidx} \} = \frac{1}{{4\coefflen}} \mathbf{1} \mbox{ for any } \dicidx \in [L]. 
\end{equation} 
Moreover, for two distinct indices $\dicidx, \dicidx' \in [L]$, we have 
\begin{equation}
\label{equ_lower_bound_squared_frob_distance} 
\| \mathbf{D}_{2,\dicidx} - \mathbf{D}_{2,\dicidx'} \|_{\rm{F}}^{2}  \stackrel{\eqref{equ_def_d_2_d_1_colwise}}{=} \| \mathbf{D}_{1,\dicidx} - \mathbf{D}_{1,\dicidx'} \|_{\rm{F}}^{2} 
\stackrel{\eqref{equ_min_distance_D_1}}{\geq} {1/40}.
\end{equation} 

Consider the matrices $\mathbf{D}_{\dicidx}$,  
\vspace*{-1mm}
\begin{equation}
\label{equ_def_matrix_D_prob_construction}
\mathbf{D}_{\dicidx} = \sqrt{1\!-\!\varepsilon'/({4\coefflen})} \mathbf{D}_{0} + \sqrt{\varepsilon'} \mathbf{D}_{2,\dicidx},
\vspace*{-1mm}
\end{equation} 
where $\dicidx \in [L]$ and 
\begin{equation} 
\label{equ_def_vareps_prime}
\varepsilon' \triangleq {320} \varepsilon.
\end{equation} 
The construction \eqref{equ_def_matrix_D_prob_construction} is feasible, since \eqref{equ_cond_epsilon_exist_dic_set} guarantees {$\varepsilon' \leq r^{2} \leq 4\coefflen$}.
We will now verify that the matrices $\mathbf{D}_{\dicidx}$, for $\dicidx \in [L]$, belong to ${\mathcal{X}}(\mathbf{D}_{0},r)$ {and moreover are such that \eqref{equ_cond_frob_norm_diff_delta} and \eqref{equ_desirata_av_mutual_info_y_index}, with $\eta$ given in \eqref{equ_bound_Q_D_0}, is satisfied.}

\emph{$\mathbf{D}_{\dicidx}$ belongs to ${\mathcal{X}}(\mathbf{D}_{0},r)$:}
Consider the $j$th column $\mathbf{d}_{\dicidx,j}$, $\mathbf{d}_{0,j}$ and $\mathbf{d}_{2,\dicidx,j}$ of $\mathbf{D}_{\dicidx}$, $\mathbf{D}_{0}$ and $\mathbf{D}_{2,\dicidx}$, respectively. {Then} 
\begin{align}
\vspace*{-8mm}
\| \mathbf{d}_{\dicidx,j} \|^{2}_{2}  & \stackrel{\eqref{equ_col_D_2_orthog_norm1_p},{\eqref{equ_def_matrix_D_prob_construction}}}{=}   (1-\varepsilon'/{(4\coefflen)}) \|\mathbf{d}_{0,j}\|_{2}^{2} +  \varepsilon' \| \mathbf{d}_{2,\dicidx,j} \|^{2}_{2}  \stackrel{\eqref{equ_col_D_2_orthog_norm1_p}}{=}  (1-\varepsilon'/{(4\coefflen)}) +   (\varepsilon'/{(4\coefflen)}) = 1. 
\end{align}
Thus, the columns of any $\mathbf{D}_{l}$, for $l \in [L]$, have unit norm. 
Moreover, 
\begin{align}
\| \mathbf{D}_{\dicidx} - \mathbf{D}_{0} \|^{2}_{\rm{F}}  & \stackrel{\eqref{equ_def_matrix_D_prob_construction}}{=} \| (1\!-\!\sqrt{1\!-\! \varepsilon'/{(4\coefflen)}})\mathbf{D}_{0} - \sqrt{\varepsilon'}\mathbf{D}_{2,\dicidx}  \|^{2}_{\rm{F}}   \nonumber \\[3mm]
& \stackrel{\eqref{equ_col_D_2_orthog_norm1_p}}{=}(1\!-\!\sqrt{1\!-\! \varepsilon'/{(4\coefflen)}})^{2} \|  \mathbf{D}_{0}  \|^{2}_{\rm{F}} \!+\!  \varepsilon' \| \mathbf{D}_{2,\dicidx}  \|^{2}_{\rm{F}} \nonumber \\[3mm] 
& \stackrel{\eqref{equ_col_D_2_orthog_norm1_p}}{=}(1\!-\!\sqrt{1\!-\! \varepsilon'/{(4\coefflen)}})^{2} \|  \mathbf{D}_{0}  \|^{2}_{\rm{F}} \!+\!  \varepsilon'{/4} \nonumber \\[3mm] 
& {\stackrel{\varepsilon'/(4\coefflen)\leq1,\mD_{0} \in \mathcal{D}}{\leq}} (\varepsilon'/{(4\coefflen)})^{2} \coefflen + \varepsilon'{/4}  \nonumber \\[3mm] 
& \stackrel{\varepsilon'/{(4\coefflen)}\leq 1}{\leq}  {(1/2)\varepsilon'} \nonumber \\[3mm] 
& \stackrel{\eqref{equ_cond_epsilon_exist_dic_set}}{\leq} r^{2}. \nonumber
\end{align} 

\emph{Lower bounding $\| \mathbf{D}_{\dicidx} - \mathbf{D}_{\dicidx'} \|_{\text{F}}^{2}$:}
The squared distance between two different matrices $\mathbf{D}_{\dicidx}$ and $\mathbf{D}_{\dicidx'}$ is obtained as 
\begin{align}
\label{equ_squared_frob_norm_diff_dic_general_1}
\| \mathbf{D}_{\dicidx} - \mathbf{D}_{\dicidx'} \|_{\text{F}}^{2} 
& \stackrel{\eqref{equ_def_matrix_D_prob_construction}}{=} \varepsilon'  \| \mathbf{D}_{2,\dicidx} - \mathbf{D}_{2,\dicidx'}  \|_{\rm{F}}^{2} \nonumber \\[3mm]
& \stackrel{\eqref{equ_lower_bound_squared_frob_distance}}{\geq}    {\varepsilon' /40}.
\end{align} 
Thus, we have verified 
\begin{equation}
\| \mathbf{D}_{\dicidx} \!-\! \mathbf{D}_{\dicidx'} \|_{\text{F}}^{2}  \geq {\varepsilon' /40}  \stackrel{\eqref{equ_def_vareps_prime}}{=}  8 \varepsilon, 
\end{equation} 
for any two different $\dicidx, \dicidx' \in [L]$. 

\emph{Upper bounding $\condmiyandlgivenx$:}
We will now upper bound the conditional MI $\condmiyandlgivenx$, conditioned on the side information $\mathbf{T}(\mX)\!=\!\mX$, between the observation $\mathbf{Y}$ and the index $\dicidx$ of the true dictionary $\mathbf{D} = \mathbf{D}_{\dicidx} \in \mathcal{D}_{0}$ in \eqref{equ_linear_model}. Here, the random index $\dicidx$ is taken u.a.r. from the set $[L]$. 
First, note that the dictionaries $\mathbf{D}_{\dicidx}$ given by \eqref{equ_def_matrix_D_prob_construction}, satisfy 
\begin{align}
\| \mathbf{D}_{\dicidx} - \mathbf{D}_{\dicidx'} \|_{\rm{F}}^{2} &\stackrel{\eqref{equ_def_matrix_D_prob_construction}}{=} \varepsilon'  \| \mathbf{D}_{2,\dicidx} - \mathbf{D}_{2,\dicidx'}  \|_{\rm{F}}^{2} \nonumber \\[3mm]
& \leq \varepsilon' \big(  \| \mathbf{D}_{2,\dicidx}  \|_{\rm{F}} +  \|  \mathbf{D}_{2,\dicidx'}  \|_{\rm{F}} \big)^{2} \nonumber \\[3mm]
& =  4  \varepsilon' \|   \mathbf{D}_{2,\dicidx} \|^{2}_{\rm{F}} \nonumber \\[3mm]
& \stackrel{\eqref{equ_col_D_2_orthog_norm1_p}{, \eqref{equ_def_vareps_prime}}}{=} {320 \varepsilon}.  \label{equ_upper_bound_squared_norm_diff_dics}
\end{align}

According to our observation model \eqref{equ_linear_model}, conditioned on the coefficients $\mathbf{x}_{\obsidx}$, the observations $\mathbf{y}_{\obsidx}$ follow a 
multivariate Gaussian distribution with covariance matrix $\sigma^{2} \mathbf{I}$ and mean vector $\mathbf{D} \mathbf{x}_{\obsidx}$. Therefore, we can employ 
a standard argument {based on the convexity of the Kullback-Leibler (KL) divergence (see, e.g., \cite{Wain2009TIT})} to upper bound $\condmiyandlgivenx$ as 
\begin{equation} 
\label{equ_upper_bound_condmi_given_x_expect_kl}
\condmiyandlgivenx \leq \frac{1}{L^{2}} \sum_{\dicidx,\dicidx' \in [L]} \expect_{{\mX}}  \{ D(f_{\mathbf{D}_{\dicidx}}(\mathbf{Y}|\mX) || f_{\mathbf{D}_{\dicidx'}}(\mathbf{Y}|\mX)) \},
\end{equation}
where $D(f_{\mathbf{D}_{\dicidx}}(\mathbf{Y}|\mathbf{X})||f_{\mathbf{D}_{\dicidx'}}(\mathbf{Y}|\mathbf{X}))$ denotes the KL divergence between the conditional probability density functions (given the coefficients $\coeffmtx\!=\!\big(\coeffvec_{1},\ldots,\coeffvec_{\samplesize}\big)$) 
of the observations $\mathbf{Y}$ for the true dictionary being either $\mathbf{D}_{\dicidx}$ or $\mathbf{D}_{\dicidx'}$.
Since, given the coefficients $\mathbf{X}$, the observations $\mathbf{y}_{\obsidx}$ are independent multivariate Gaussian random vectors {with mean $\mD \vx_{\obsidx}$ 
and the same covariance matrix $\sigma^{2} \mathbf{I}_{\measlen}$}, we can apply the formula \cite[Eq. (3)]{DurrieuKelly2012} {for the KL-divergence} to obtain 
\begin{align}
D(f_{\mathbf{D}_{\dicidx}}(\mathbf{Y}|\mathbf{X})||f_{\mathbf{D}_{\dicidx'}}(\mathbf{Y}|\mathbf{X})) & =  \sum_{\obsidx \in [\samplesize]} \frac{1}{2\sigma^{2}} \| (\mathbf{D}_{\dicidx} \!-\! \mathbf{D}_{\dicidx'}) \mathbf{x}_{\obsidx} \|^{2}  \nonumber \\
& =  \sum_{\obsidx \in [\samplesize]} \frac{1}{2\sigma^{2}} \mbox{Tr} \big\{  (\mathbf{D}_{\dicidx} \!-\! \mathbf{D}_{\dicidx'})^{T}  (\mathbf{D}_{\dicidx} \!-\! \mathbf{D}_{\dicidx'}) \mathbf{x}_{\obsidx} \mathbf{x}_{\obsidx}^{T} \big\}. 
\label{eq_expr_kl_div_pdf_y_condX_D_l}
\end{align}
Inserting \eqref{eq_expr_kl_div_pdf_y_condX_D_l} into \eqref{equ_upper_bound_condmi_given_x_expect_kl} and using \eqref{equ_upper_bound_squared_norm_diff_dics} as well as 
\begin{equation}
\mbox{Tr} \{ \mathbf{A}^{T} \mathbf{A} \mathbf{\Sigma}_{x} \} \leq  \| \mathbf{\Sigma}_{x} \|_{2} \| \mathbf{A} \|_{\rm{F}}^{2}, \nonumber 
\end{equation} 
yields 
\begin{equation}
\condmiyandlgivenx \leq \frac{{320} \samplesize \| \mathbf{\Sigma}_{x} \|_{2}\varepsilon}{\sigma^{2}} 
\end{equation}  
{completing the proof.}
\end{proof}

For the proof of Theorem \ref{thm_main_result_sparse_coeff} we will need a variation of Lemma \ref{lem_existence_dicationary_set_desiredata}, which is based on using the side information $\mathbf{T}(\mathbf{X}) \!=\!\supp(\mathbf{X})$ instead of $\mathbf{X}$ itself. 

\begin{lemma}
\label{lem_existence_dicationary_set_desiredata_support_sideinfo}
Consider a DL problem based on the {generative} model \eqref{equ_linear_model} such that \eqref{eq_dic_belongs_to_small_ball} holds with some {$r \leq 2\sqrt{\coefflen}$}. The random sparse coefficients $\vx$ are distributed 
according to \eqref{equ_model_sparse_coeffs} with {${\rm SNR} \leq (1/(9 \sqrt{80})) \measlen/(2\sparsity)$}. We assume that the reference dictionary $\mD_{0}$ satisfies {the} RIP of order $\sparsity$ with constant $\delta_{\sparsity} \leq 1/2$.

If $(\measlen \!-\! 1)\coefflen \geq {50}$ then 
there exists a set $\mathcal{D}_{0} \subseteq \mathcal{D}$ of cardinality $L = 2^{(\measlen\!-\!1)\coefflen/5}$ such that 
 \eqref{equ_cond_frob_norm_diff_delta} and \eqref{equ_desirata_av_mutual_info_y_index}, for the side information $\mathbf{T}(\mathbf{X})\!=\!\supp(\mathbf{X})$, are satisfied with
\begin{equation}
\label{equ_bound_Q_D_0support_side_info}
\eta = {12960 \samplesize {\rm SNR}^{2} \measlen^{2}  \varepsilon/\coefflen},
\end{equation}
and 
\begin{equation} 
\label{equ_cond_epsilon_exist_dic_set_support_side_info}
\varepsilon \leq  {r^{2}/(320 \sparsity)}.
\end{equation} 
\end{lemma}

\begin{proof}
We will use the same ensemble $\mathcal{D}_{0}$ (cf.\ \eqref{equ_def_matrix_D_prob_construction}) as in the proof of Lemma \ref{lem_existence_dicationary_set_desiredata} (note that condition \eqref{equ_cond_epsilon_exist_dic_set_support_side_info} implies \eqref{equ_cond_epsilon_exist_dic_set} since $\sparsity \geq 1$).
Thus, we already verified in the proof of Lemma \ref{lem_existence_dicationary_set_desiredata} that $\mathcal{D}_{0} \subseteq {\mathcal{X}}(\mathbf{D}_{0},r)$ and \eqref{equ_cond_frob_norm_diff_delta} is satisfied. 

\emph{Upper bounding $\condmiyandlgivenx$:}
We will now upper bound the conditional MI $\condmiyandlgivenx$, conditioned on the side information $\mathbf{T}(\mathbf{X})=\supp(\mathbf{X})$, between the observation $\mathbf{Y}\!=\!\big(\mathbf{y}_{1},\ldots,\mathbf{y}_{\samplesize}\big)$ and the index $\dicidx$ of the true dictionary $\mathbf{D} = \mathbf{D}_{\dicidx} \in \mathcal{D}_{0}$ in \eqref{equ_linear_model}. Here, the random index $\dicidx$ is taken u.a.r. from the set $[L]$ and the 
conditioning is w.r.t. the random supports $\supp(\mathbf{X}) \!=\! \big( \supp(\mathbf{x}_{1}),\ldots,\supp(\mathbf{x}_{\samplesize}) \big)$ of the coefficient vectors $\mathbf{x}_{\obsidx}$, being i.i.d. realizations 
of the sparse vector $\mathbf{x}$ given by \eqref{equ_model_sparse_coeffs}. Let us introduce for the following the shorthand $\mathcal{S}_{\obsidx} \defeq \supp(\mathbf{x}_{\obsidx})$. 

Note that, conditioned on {$\mathcal{S}_{\obsidx}$}, the columns of the matrix $\mathbf{Y}$, i.e., the observed samples $\mathbf{y}_{\obsidx}$ are independent multivariate Gaussian random vectors with zero mean and 
covariance matrix 
\begin{equation} 
\mathbf{\Sigma}_{\obsidx}  = \sigma_{a}^{2} \mathbf{D}_{\mathcal{S}_{\obsidx}}  \mathbf{D}^{T}_{\mathcal{S}_{\obsidx}}+ \sigma^{2} \mathbf{I}. 
\end{equation}

Thus, according to \cite[Eq. (18)]{WangWain2010}, we can use the following bound on the conditional MI
\begin{equation} 
\label{equ_KL_based_bound_MI}
\condmiyandlgivenx \leq  \expect_{\mathbf{T}(\mathbf{X})} \bigg \{ \sum_{\obsidx \in [\samplesize]} (1/L^2) \sum_{\dicidx,\dicidx' \in [L]}  \mbox{Tr} \big\{ \big[ \mathbf{\Sigma}^{-1}_{\obsidx,\dicidx} - \mathbf{\Sigma}^{-1}_{\obsidx,\dicidx'}  \big] 
\big[  \mathbf{\Sigma}_{\obsidx,\dicidx'} - \mathbf{\Sigma}_{\obsidx,\dicidx}  \big] \big\}  \bigg \} 
\end{equation} 
with
\begin{equation}
\label{equ_def_sigma_dicidx_mathcals}  
\mathbf{\Sigma}_{\obsidx,\dicidx} \defeq \sigma_{a}^{2} \mathbf{D}_{\dicidx,\mathcal{S}_{\obsidx}}\mathbf{D}^{T}_{\dicidx,\mathcal{S}_{\obsidx}} + \sigma^{2} \mathbf{I}.
\end{equation} 
Here, $\expect_{\mathbf{T}(\mathbf{X})} \big\{  \cdot \big \}$ denotes expectation with respect to the side information $\mathbf{T}(\coeffmtx)\!=\!\big({\mathcal{S}_{1}},\ldots,{\mathcal{S}_{\samplesize}}\big)$ which is distributed uniformly over the $\samplesize$-fold product $\Xi \times \ldots \times \Xi$ (cf.\ \eqref{equ_sparse_coeff_support_prob}). 
Since any of the matrices $\mathbf{\Sigma}_{\obsidx,\dicidx}$ is made up of the common component $\sigma^{2} \mathbf{I}$ and the individual component $\sigma_{a}^{2} \mathbf{D}_{\dicidx,\mathcal{S}_{\obsidx}}\mathbf{D}^{T}_{\dicidx,\mathcal{S}_{\obsidx}}$, which has 
rank not larger than $\sparsity$, for any two $\dicidx, \dicidx' \in [L]$, the difference $\mathbf{\Sigma}_{\obsidx,\dicidx}\!-\!\mathbf{\Sigma}_{\obsidx,\dicidx'}$ satisfies 
\begin{equation} 
\label{equ_rank_leq_2s}
\rank \big\{ \mathbf{\Sigma}_{\obsidx,\dicidx}\!-\!\mathbf{\Sigma}_{\obsidx,\dicidx'}  \big\} \leq 2 \sparsity. 
\end{equation} 
Therefore, using $\mbox{Tr}\{ \mathbf{A} \} \leq \rank \{ \mathbf{A} \} \| \mathbf{A} \|_{2}$ and \eqref{equ_rank_leq_2s}, we can rewrite \eqref{equ_KL_based_bound_MI} as 
\begin{equation} 
\label{equ_upper_bound_cmi_kl_div}
\condmiyandlgivenx \leq  2 \sparsity \expect_{\mathbf{T}(\mathbf{X})} \bigg \{ \sum_{\obsidx \in [\samplesize]} \frac{1}{L^{2}} \sum_{\dicidx, \dicidx' \in [L]} \big\| \mathbf{\Sigma}^{-1}_{\obsidx,\dicidx} \!-\! \mathbf{\Sigma}^{-1}_{\obsidx,\dicidx'} \big\|_{2} \big\| \mathbf{\Sigma}_{\obsidx,\dicidx'} \!-\! \mathbf{\Sigma}_{\obsidx,\dicidx} \big\|_{2}\bigg \}.  
\end{equation}

In what follows, we will first upper bound the spectral norm $\big\| \mathbf{\Sigma}_{\obsidx,\dicidx'} \!-\! \mathbf{\Sigma}_{\obsidx,\dicidx} \big\|_{2}$ and subsequently, using a perturbation result \cite{golub96} for matrix inversion, upper bound the spectral norm  $\big\| \mathbf{\Sigma}^{-1}_{\obsidx,\dicidx} \!-\! \mathbf{\Sigma}^{-1} _{\obsidx,\dicidx'} \big\|_{2}$. Inserting these two bounds into \eqref{equ_upper_bound_cmi_kl_div} will then yield the final upper bound on $\condmiyandlgivenx$.

{Due to the construction \eqref{equ_def_matrix_D_prob_construction},}
\begin{align}
\label{equ_block_structure_Sigma_l_i}
\mathbf{\Sigma}_{\obsidx,\dicidx} \!-\! \mathbf{\Sigma}_{\obsidx,\dicidx'}  \! & \stackrel{\eqref{equ_def_sigma_dicidx_mathcals}}{=} \! \sigma_a^{2} \big(  \mathbf{D}_{\dicidx,\mathcal{S}_{\obsidx}}\mathbf{D}^{T}_{\dicidx,\mathcal{S}_{\obsidx}} \!-\!  \mathbf{D}_{\dicidx',\mathcal{S}_{\obsidx}}\mathbf{D}^{T}_{\dicidx',\mathcal{S}_{\obsidx}} \big) \nonumber \\[4mm]
&  \stackrel{\eqref{equ_def_matrix_D_prob_construction}}{=}  \sigma_{a}^{2} \sqrt{1\!-\!\varepsilon'/{4\coefflen}}\sqrt{\varepsilon'}  \big( \overline{\mathbf{D}_{0,\mathcal{S}_{\obsidx}} \mathbf{D}^{{T}}_{2,\dicidx,\mathcal{S}_{\obsidx}}} \!-\! \overline{ \mathbf{D}_{0,\mathcal{S}_{\obsidx}} \mathbf{D}^{{T}}_{2,\dicidx',\mathcal{S}_{\obsidx}}} \big) + \sigma_{a}^{2} \varepsilon' (\mathbf{D}_{2,\dicidx,\mathcal{S}_{\obsidx}} \mathbf{D}^{T}_{2,\dicidx,\mathcal{S}_{\obsidx}} \!-\! 
\mathbf{D}_{2,\dicidx',\mathcal{S}_{\obsidx}}\mathbf{D}^{T}_{2,\dicidx',\mathcal{S}_{\obsidx}})
\end{align} 
with the shorthand $\overline{\mathbf{X}} \defeq \mathbf{X} + \mathbf{X}^{T}$. 
In what follows, we need 
\begin{equation}
\label{equ_spec_norm_bounds_side_info_supp}
\| \mathbf{D}_{0,\mathcal{S}} \|_{2} \leq \sqrt{3/2} \mbox{, }  \| \mathbf{D}_{2,\dicidx,\mathcal{S}}\|_{2} \leq \sqrt{\sparsity/{(4\coefflen)}}  \mbox{, and } \| \mathbf{\Sigma}^{-1}_{\obsidx,\dicidx} \|_{2} \leq 1/\sigma^{2}, 
\end{equation} 
for any $\dicidx \in [L]$ and any subset $\mathcal{S} \subset [\coefflen]$ with $| \mathcal{S} | \leq \sparsity$. 
The first bound in \eqref{equ_spec_norm_bounds_side_info_supp} follows from the assumed RIP {(with constant $\delta_{\sparsity}\leq1/2$)} of the reference dictionary $\mathbf{D}_{0}$. 
The second bound in \eqref{equ_spec_norm_bounds_side_info_supp} is valid because the matrices $\mathbf{D}_{2,\dicidx}$ have columns with norm equal to $1/\sqrt{{4}\coefflen}$ (cf.\ \eqref{equ_col_D_2_orthog_norm1_p}). 
For the verification of the last bound in \eqref{equ_spec_norm_bounds_side_info_supp} we note that, according to \eqref{equ_def_sigma_dicidx_mathcals}, $\lambda_{\text{min}}(\mathbf{\Sigma}_{\obsidx,\dicidx}) \geq \sigma^{2}$. 
{Therefore, }
\begin{align} 
\label{equ_norm_bound_diff_sigmas_l_l_prime}
\big\| \mathbf{\Sigma}_{\obsidx,\dicidx} \!-\! \mathbf{\Sigma}_{\obsidx,\dicidx'} \|_{2} 
&\stackrel{\eqref{equ_block_structure_Sigma_l_i},\eqref{equ_spec_norm_bounds_side_info_supp}}{{\leq}} {2\sqrt{3/2}} \sigma_{a}^{2}  \sqrt{1\!-\!\varepsilon'/({4\coefflen})}\sqrt{\varepsilon'}  \sqrt{\sparsity/({4\coefflen})} \!+\! {2} \sigma_{a}^{2} \varepsilon' \sparsity/({4\coefflen}) \nonumber \\[3mm]
& \stackrel{{\eqref{equ_cond_epsilon_exist_dic_set_support_side_info}}}{\leq} 4.5 \sigma_{a}^{2} \sqrt{\varepsilon' \sparsity/({4\coefflen})}.
\end{align} 

{Since the true dictionary $\mD$ is assumed to satisfy the RIP with constant $\delta_{\sparsity}\leq1/2$, the low SNR condition ${\rm SNR} \leq \measlen/(2\sparsity)$ implies 
via \eqref{equ_doulbe_inequ_SNR}, 
\begin{equation}
\label{equ_low_SNR_variances}
(\sigma_{a}/\sigma)^2 \leq  \frac{1}{9 \sqrt{80}}. 
\end{equation}}
Since 
\begin{align} 
\big\| \mathbf{\Sigma}^{-1}_{\obsidx,\dicidx} \big(\mathbf{\Sigma}_{\obsidx,\dicidx} \!-\! \mathbf{\Sigma}_{\obsidx,\dicidx'} \big) \big\|_{2} 
& \leq  \big\| \mathbf{\Sigma}^{-1}_{\obsidx,\dicidx} \big\|_{2}  \big\|  \mathbf{\Sigma}_{\obsidx,\dicidx} \!-\! \mathbf{\Sigma}_{\obsidx,\dicidx'}   \big\|_{2}  \nonumber \\[3mm] 
& \stackrel{\eqref{equ_spec_norm_bounds_side_info_supp},\eqref{equ_norm_bound_diff_sigmas_l_l_prime}}{\leq}  4.5 (\sigma_{a}/\sigma)^2 \sqrt{\varepsilon' \sparsity/\coefflen} \nonumber \\[3mm] 
& \stackrel{{\eqref{equ_low_SNR_variances},\eqref{equ_cond_epsilon_exist_dic_set_support_side_info}}}{{\leq}} 1/2,  
\end{align} 
we can invoke \cite[Theorem 2.3.4.]{golub96} yielding 
\begin{equation} 
\label{equ_bound_diff_prec_matrices_roof_ensemble_dic}
\big\| \mathbf{\Sigma}^{-1}_{\obsidx,\dicidx} \!-\! \mathbf{\Sigma}^{-1}_{\obsidx,\dicidx'} \big\|_{2} \leq 2  \big\| \mathbf{\Sigma}_{\obsidx,\dicidx}^{-1} \big\|_{2}^{2} \big\| \mathbf{\Sigma}_{\obsidx,\dicidx} \!-\! \mathbf{\Sigma}_{\obsidx,\dicidx'} \big\|_{2} 
\stackrel{{\eqref{equ_spec_norm_bounds_side_info_supp}}}{\leq} 2 \sigma^{-4} \big\|  \mathbf{\Sigma}_{\obsidx,\dicidx'} \!-\! \mathbf{\Sigma}_{\obsidx,\dicidx}   \big\|_{2}.
\end{equation} 
Inserting \eqref{equ_norm_bound_diff_sigmas_l_l_prime} and \eqref{equ_bound_diff_prec_matrices_roof_ensemble_dic} into \eqref{equ_upper_bound_cmi_kl_div} yields the bound 
\begin{align} 
\condmiyandlgivenx & \leq 4 \samplesize \sparsity \sigma^{-4} (1/L^2) \sum_{\dicidx,\dicidx \in [L]} \big\| \mathbf{\Sigma}_{\obsidx,\dicidx'} \!-\! \mathbf{\Sigma}_{\obsidx,\dicidx}  \big\|^{2}_{2}  \nonumber \\[1mm]
& \stackrel{\eqref{equ_norm_bound_diff_sigmas_l_l_prime}}{\leq}  4 \cdot 4.5^2 \samplesize \sparsity^2 (\sigma_{a}/\sigma)^{4} \varepsilon'/{(4\coefflen)} \nonumber \\[3mm]
&  \stackrel{\varepsilon' = {320} \varepsilon}{\leq}  {6480} \samplesize \sparsity^2 (\sigma_{a}/\sigma)^{4} \varepsilon/\coefflen \nonumber  \\[3mm]
& {\stackrel{\delta_{\sparsity}\leq1/2, \eqref{equ_doulbe_inequ_SNR}}{\leq} 12960 \samplesize {\rm SNR}^{2} \measlen^{2}  \varepsilon/\coefflen},
\vspace*{-3mm}
\end{align} 
{completing the proof.}
\vspace*{-3mm}
\end{proof}

The next result relates the cardinality $L$ of a subset $\mathcal{D}_{0} \!=\! \{\mathbf{D}_{1},\ldots,\mathbf{D}_{L} \} \subseteq \mathcal{D}$ to the conditional MI $\condmiyandlgivenx$ between the observation $\mathbf{Y}\!=\!\big(\vy_{1},\ldots,\vy_{\samplesize}\big)$, with $\vy_{\obsidx}$ i.i.d.\ according to \eqref{equ_linear_model}, and a random index $\dicidx$ selecting the true dictionary $\mathbf{D}$ in \eqref{equ_linear_model} u.a.r. from $\mathcal{D}_{0}$.
\begin{lemma}
\label{lem_inequality_Q_log_L}
Consider the DL problem \eqref{equ_linear_model} with minimax risk $\minimaxrisk$  (cf.\ \eqref{equ_def_minimax_problem}){, which is assumed to be upper bounded by a positive number $\varepsilon$, i.e., $\minimaxrisk \leq \varepsilon$.
Assume there exisits a finite set} $\mathcal{D}_{0} = \{ \mathbf{D}_{1}, \ldots, \mathbf{D}_{L} \} \subseteq \mathcal{D}$ consisting of $L$ distinct dictionaries $\mathbf{D}_{\dicidx} \in \mathbb{R}^{\measlen \times \coefflen}$ such that 
\begin{equation}
\label{equ_cond_Frob_norm_squared_greater_8}
\| \mathbf{D}_{\dicidx} \!-\! \mathbf{D}_{\dicidx'} \|^{2}_{\text{F}} \geq 8 \bar{\delta}_{\dicidx,\dicidx'} \varepsilon.
\end{equation}

Then, {for any function $\mathbf{T}(\mX)$ of the true coefficients $\mX=\big(\vx_{1},\ldots,\vx_{\samplesize}\big)$}, 
\begin{equation}
\condmiyandlgivenx \geq  (1/2) \log_{2}(L) -1. 
\end{equation} 
\end{lemma}

\begin{proof} 
Our proof idea closely follows those of \cite[Thm.\ 1]{CandesDavenport2013}.
Consider a minimax estimator $\widehat{\mathbf{D}}(\mathbf{Y})$, whose worst case MSE is equal to $\minimaxrisk$, i.e., 
\begin{equation}
 \sup_{\mathbf{D} \in \mathcal{D}}  \expect_{{\mY}}  \big\{ \| \widehat{\mathbf{D}}(\mathbf{Y}) \!-\! \mathbf{D} \|^{2}_{\text{F}} \big\} = \minimaxrisk, 
\end{equation}
and, in turn since $\mathcal{D}_{0} \subseteq \mathcal{D}$, 
\begin{equation}
\label{equ_worst_case_risk_minimax_est_proof_lemma}
\sup_{\mathbf{D} \in \mathcal{D}_{0}}  \expect_{{\mY}}  \big\{ \| \widehat{\mathbf{D}}(\mathbf{Y}) \!-\! \mathbf{D} \|^{2}_{\text{F}} \big\} \leq \minimaxrisk.
\end{equation} 

Based on the estimator $\widehat{\mathbf{D}}(\mathbf{Y})$, we define a detector $\hat{l}(\mathbf{Y})$ for the index of 
true underlying dictionary $\mathbf{D}_{\dicidx} \in \mathcal{D}_{0}$ via 
\begin{equation}
\label{equ_def_detector_min_forb_norm_squared}
\hat{l}(\mY) \triangleq \argmin_{\dicidx' \in [L]}  \| \mD_{\dicidx'} \!-\! \widehat{\mD}(\mY) \|^{2}_{\text{F}}. 
\end{equation}
{In case of ties, i.e., when there are multiple indices $\dicidx'$ such that $\mD_{\dicidx'}$ achieves the minimum in \eqref{equ_def_detector_min_forb_norm_squared}, we randomly select one of the minimizing indices as the estimate $\hat{l}(\mY)$.}
Let us now assume that the index $l$ is selected u.a.r. from $[L]$ and bound the probability $P_{e}$ of a detection error, i.e., $P_{e} \triangleq \prob \{ \hat{l}(\mY) \neq l \}$. 
Note that if 
\begin{equation}
\label{equ_assume_hat_D_D_l_lower_2_minmax}
 \| \widehat{\mathbf{D}}(\mathbf{Y}) \!-\! \mathbf{D}_{\dicidx} \|^{2}_{\text{F}} <2 \varepsilon 
\end{equation}
then for any wrong index $\dicidx' \in [L] \setminus \{l\}$, 
\begin{align}
\| \widehat{\mathbf{D}}(\mathbf{Y}) \!-\! \mathbf{D}_{\dicidx'} \|_{\text{F}} & = \| \widehat{\mD}(\mY) \!-\! \mD_{\dicidx'} \!+\! \mathbf{D}_{\dicidx} \!-\! \mathbf{D}_{\dicidx} \|_{\text{F}} \nonumber \\[3mm] 
& \geq  \| \mathbf{D}_{\dicidx} \!-\! \mathbf{D}_{\dicidx'} \|_{\text{F}} -  \| \widehat{\mathbf{D}}(\mY) - \mathbf{D}_{\dicidx} \|_{F}  \nonumber \\[3mm]
& \stackrel{\eqref{equ_cond_Frob_norm_squared_greater_8},\eqref{equ_assume_hat_D_D_l_lower_2_minmax}}{\geq}  (\sqrt{8} - \sqrt{2}) {\sqrt{\varepsilon}} \nonumber \\[3mm]
& {=} {\sqrt{2 \varepsilon}} \nonumber \\[3mm]
& \stackrel{\eqref{equ_assume_hat_D_D_l_lower_2_minmax}}{>}  \| \widehat{\mathbf{D}}(\mY) - \mathbf{D}_{\dicidx} \|_{\text{F}}. 
\end{align} 
Thus, the condition \eqref{equ_assume_hat_D_D_l_lower_2_minmax} guarantees that the detector $\hat{l}(\mathbf{Y})$ in \eqref{equ_def_detector_min_forb_norm_squared} delivers the correct index $l$. 
Therefore, in turn, a detection error can only occur if $ \| \widehat{\mathbf{D}} - \mathbf{D}_{\dicidx} \|^{2}_{\text{F}} \geq  2 \varepsilon$ implying that 
\begin{align}
\label{equ_P_e_lower_eq_1_2_relation_L_symm_KL_div}
P_{e} & \leq \prob \big\{ \| \widehat{\mathbf{D}}(\mY) - \mathbf{D}_{\dicidx} \|^{2}_{\text{F}} \geq  2  \varepsilon \big\} \nonumber \\[3mm]
& \stackrel{(a)}{\leq}  \frac{1}{ 2  \varepsilon }   \expect_{{\mY}}  \big\{\| \widehat{\mathbf{D}}(\mY) - \mathbf{D}_{\dicidx} \|^{2}_{\text{F}}  \big\}  \nonumber \\[3mm]
&  \stackrel{\eqref{equ_worst_case_risk_minimax_est_proof_lemma}}{\leq} \frac{\minimaxrisk}{ 2 \varepsilon} \nonumber \\[3mm]
& {\stackrel{\minimaxrisk \leq \varepsilon}{\leq}} 1/2,
\end{align}
where $(a)$ is due to the Markov inequality{\cite{BillingsleyProbMeasure}}.
However, according to Lemma \ref{lem_fano_inequ_first}, {we also have}
\begin{equation}
\condmiyandlgivenx \geq \log_{2}(L) - P_{e}  \log_{2}(L)-1,  
\end{equation} 
and, in turn, since $P_{e} \leq 1/2$ by \eqref{equ_P_e_lower_eq_1_2_relation_L_symm_KL_div}, 
\begin{equation}
\label{equ_MI_cond_i_geq_1_1_2_log_L}
\condmiyandlgivenx \geq (1/2) \log_{2}(L)-1,  \nonumber
\end{equation}
{completing the proof.}
\end{proof}

Finally, we simply have to put the pieces together to obtain Theorem \ref{thm_main_result} and Theorem \ref{thm_main_result_sparse_coeff}. 

\emph{Proof of Theorem \ref{thm_main_result}:} 
According to Lemma \ref{lem_existence_dicationary_set_desiredata}, if $(\measlen\!-\!1)\coefflen\!\geq\!{50}$ and for any {$\varepsilon\!<\!r^2/320$ (this condition is implied by the first bound in \eqref{equ_conditions_s_2_first_charac_minimax_risk})}, there exists a set $\mathcal{D}_{0} \subseteq {\mathcal{X}}(\mathbf{D}_{0},r)$ of cardinality $L \!=\! 2^{(\measlen\!-\!1)\coefflen/5}$ 
satisfying \eqref{equ_cond_frob_norm_diff_delta} and \eqref{equ_desirata_av_mutual_info_y_index} with 
$\eta \!=\! {320}  \samplesize \| \mathbf{\Sigma}_{x} \|_{2}\varepsilon /\sigma^2 $. 
Applying Lemma \ref{lem_inequality_Q_log_L} to the set $\mathcal{D}_{0}$ yields, in turn, 
\begin{equation} 
{320} \samplesize  \| \mathbf{\Sigma}_{x} \|_{2}\varepsilon /\sigma^2   \geq \condmiyandlgivenx \geq (1/2) \log_{2}(L) -1 
\end{equation}
implying 
\begin{equation}
\varepsilon  \geq  \frac{\sigma^{2}}{{320} \samplesize \| \mathbf{\Sigma}_{x} \|_{2}}  ((1/2) \log_{2}(L) -1) \geq \frac{\sigma^2}{{320} \samplesize\| \mathbf{\Sigma}_{x} \|_{2}  }  ((\measlen\!-\!1)\coefflen/10 -1).
\end{equation}

\emph{Proof of Theorem \ref{thm_main_result_sparse_coeff}:} 
According to Lemma \ref{lem_existence_dicationary_set_desiredata_support_sideinfo}, 
if $(\measlen\!-\!1)\coefflen\!\geq\!{50}$ and for any {$\varepsilon\!<\!r^2/(320\sparsity)$ (this condition is implied by the first bound in \eqref{equ_conditions_s_2_first_charac_minimax_risk_sparse})}, there exists a set $\mathcal{D}_{0} \subseteq {\mathcal{X}}(\mathbf{D}_{0},r)$ of cardinality $L \!=\! 2^{(\measlen\!-\!1)\coefflen/5}$ 
satisfying \eqref{equ_cond_frob_norm_diff_delta} and \eqref{equ_desirata_av_mutual_info_y_index} with 
$\eta \!=\! {12960} \samplesize \measlen^2 {\rm SNR}^{2} \varepsilon/\coefflen$. 
Applying Lemma \ref{lem_inequality_Q_log_L} to the set $\mathcal{D}_{0}$ yields, in turn, 
\begin{equation} 
 {12960 \samplesize \measlen^2 {\rm SNR}^{2}} \varepsilon/\coefflen \geq \condmiyandlgivenx \geq (1/2) \log_{2}(L) -1 
\end{equation}
implying 
\begin{equation}
\varepsilon  \geq  {\frac{{\rm SNR}^{-2} \coefflen}{12960 \samplesize \measlen^{2} }  ((1/2) \log_{2}(L) -1) \geq \frac{{\rm SNR}^{-2} \coefflen}{12960 \samplesize \measlen^{2}}  ((\measlen\!-\!1)\coefflen/10 -1)}.
\end{equation}

{
\emph{Proof of Theorem \ref{thm_partial_converse}:} 
First note that any dictionary $\mD \in \mathcal{X}(\mD_{0}=\mathbf{I},r)$ can be written as  
\begin{equation} 
\label{equ_proof_partial_conv_decomp}
\mD = \mathbf{I} + {\bf \Delta} \mbox{ , with } \| {\bf \Delta}\|_{\rm F} \leq r. 
\end{equation}
Any matrix $\mD$ of the form \eqref{equ_proof_partial_conv_decomp} satisfies the RIP with constant $\delta_{\sparsity}$ such that 
\begin{equation}
\label{equ_double_bound_upper_bound}
(1-r)^2 \leq  1-\delta_{\sparsity} \leq 1 + \delta_{\sparsity} \leq (1+r)^2. 
\end{equation}
Moreover, since we assume the coefficient vectors $\vx_{k}$ in \eqref{equ_single_linear_model} to be discrete-valued $\vx_{k} \in \{-1,0,1\}^{\coefflen}$ 
and complying with \eqref{equ_sparse_coeff_support_prob}, 
\begin{equation}
\label{equ_variance_coeff_suff_cond_upper_bound}
\expect_{\vx_{k}} \{  \vx_{k,t}^{2} \}  = \sparsity/\coefflen.
\end{equation}
and 
\begin{equation}
\label{equ_norm_sqare_equal_s_upper_bound}
\| \vx_{k} \|^{2}_{2} = \sparsity.  
\end{equation} 
For \eqref{equ_norm_sqare_equal_s_upper_bound}, we used the fact that the non-zero entries of $\vx_{i}$ all have the same magnitude equal to one. 
Combining \eqref{equ_norm_sqare_equal_s_upper_bound} with \eqref{equ_double_bound_upper_bound}, we obtain the following bound on the SNR: 
\begin{equation}
\label{equ_SNR_bound_upper_bound}
{\rm SNR} = \expect_{{\vx}} \{ \| \mD \vx \|^{2}_{2} \} / \expect_{{\noisevec}} \{ \| \noisevec \|^{2}_{2} \} \stackrel{\eqref{equ_def_RIP}}{\geq} (1-\delta_{\sparsity})  \sparsity /(\measlen \sigma^{2}) \stackrel{\eqref{equ_double_bound_upper_bound}}{\geq}  (1-r)^2  \sparsity /(\measlen \sigma^{2}).  
\end{equation}  
}

{
In order to derive an upper bound on the MSE of the DL scheme given by Algorithm \ref{alg_simple}, we first split the MSE of $\widehat{\mD}(\mY)=\big(\widehat{\vd}_{1}(\mY),\ldots,\widehat{\vd}_{\coefflen}(\mY)\big)$ into a sum of the MSE for the individual columns of the dictionary, i.e., 
\begin{equation}
\label{equ_proof_upper_bound_MSE_splits_over_columns}
\expect_{{\mY}} \{ \| \widehat{\mD}{(\mY)} - \mD \|^{2}_{\rm F} \}  = \sum_{l \in [\coefflen]} \expect_{{\mY}} \{ \| \widehat{\vd}_{l}(\mY) - \vd_{l} \|^{2}_{2} \}. 
\end{equation} 
Thus, we may analyze the column-wise MSE $\expect_{{\mY}} \{ \| \widehat{\vd}_{l}{(\mY)} - \vd_{l} \|^{2}_{2} \}$ separately for each column index $l \in [\coefflen]$. 
Note that, by construction
\begin{equation}
\label{diff_col_max_2}
\|\widehat{\vd}_{l}{(\mY)} - \vd_{l} \|^{2}_{2} \leq 2,   
\end{equation} 
since the columns of $\widehat{\mD}(\mY)$ and $\mD$ have norm at most one. 
}

{
We will analyze the MSE of the DL scheme in Algorithm \ref{alg_simple} by conditioning on 
a specific event $\mathcal{C}$, defined as 
\begin{equation}
 \mathcal{C} \defeq \bigcap_{\substack{\obsidx \in [\samplesize] \\ l \in [\coefflen]}} \{ |n_{\obsidx,l}| < 0.4 \}. 
\end{equation}  
Assuming $r \sqrt{\sparsity} \leq 1/10$, the occurrence of $\mathcal{C}$ implies the estimated coefficient matrix $\widehat{\mX}$ to coincide with the true coefficients $\mX$, i.e., 
\begin{equation}
\label{equ_coeffs_perfect_recov_C}
\prob \{ \mX = \widehat{\mX} | \mathcal{C} \}=1.
\end{equation} 
Indeed, if $r \sqrt{\sparsity} \leq 1/10$ and $|n_{\obsidx,l}| < 0.4$ for every $\obsidx \in [\samplesize]$ and $l \in [\coefflen]$, then $y_{\obsidx,j} > 0.5$ if $x_{\obsidx,j} =1$ (implying $\hat{x}_{\obsidx,j}=1$), and $y_{\obsidx,j} < - 0.5$ if $x_{\obsidx,j} =-1$ (implying $\hat{x}_{\obsidx,j}=-1$) as 
well as  $|y_{\obsidx,j}| \leq 0.5$ if $x_{\obsidx,j} =0$ (implying $\hat{x}_{\obsidx,j}=0$). 
The characterization the probability of $\mathcal{C}$ is straightforward, since the noise entries $n_{\obsidx,l}$ are assumed i.i.d. Gaussian variables with 
zero mean and variance $\sigma^{2}$. In particular, the tail bound \cite[Proposition 7.5]{RauhutFoucartCS}) together with 
a union bound over all entries of the coefficient matrix $\mX\!=\!(\vx_{1},\ldots,\vx_{\samplesize})\in\mathbb{R}^{\coefflen\times\samplesize}$, yields
\begin{equation}
\prob \{ \mathcal{C}^{c} \} \leq \exp(-  \coefflen \samplesize 0.4^2/(2\sigma^2)).  
\vspace*{-3mm}
\end{equation}}

{
As a next step we upper bound the MSE using the law of total expectation: 
\begin{align}
\label{equ_proof_upper_bound_law_total_expectation}
 \expect_{{\mY}} \{ \| \widehat{\vd}_{l}{(\mY)} - \vd_{l} \|^{2}_{2} \} & = \expect_{{\mY,\mN}} \{ \| \widehat{\vd}_{l}{(\mY)} - \vd_{l} \|^{2}_{2}\big| \mathcal{C} \} \prob(\mathcal{C} ) + \expect_{{\mY,\mN}} \{\| \widehat{\vd}_{l}{(\mY)} - \vd_{l} \|^{2}_{2} \big| \mathcal{C}^{c} \} \prob(\mathcal{C}^{c})  \nonumber \\[3mm] 
 & \stackrel{\eqref{diff_col_max_2}}{\leq} \expect_{{\mY,\mN}} \{ \| \widehat{\vd}_{l}{(\mY)} - \vd_{l} \|^{2}_{2}\big| \mathcal{C} \} \prob(\mathcal{C} ) + 2 \prob(\mathcal{C}^{c}) \nonumber \\[3mm] 
 & \leq  \expect_{{\mY,\mN}} \{ \| \widehat{\vd}_{l}{(\mY)} - \vd_{l} \|^{2}_{2}\big| \mathcal{C} \} + 2  \exp(- \coefflen \samplesize 0.4^2/(2\sigma^2)).
\end{align} 
}

{The conditional MSE $\expect \{  \| \widehat{\vd}_{l}(\mY)  - \vd_{l} \|^{2}_{2} \big| \mathcal{C} \}$ can be bounded by  
\begin{align} 
\label{part_converse_cond_expect_1}
\expect_{{\mY,\mN}} \{  \| \widehat{\vd}_{l}(\mY) - \vd_{l} \|^{2}_{2} \big| \mathcal{C} \} & = \expect_{{\mY,\mN}} \big\{  \big\| \mathbf{P}_{\overline{\mathcal{B}}(\mathbf{e}_{l},\rho)} \widetilde{\vd}_{l}(\mY) - \vd_{l} \|^{2}_{2} \big| \mathcal{C}\big\} \nonumber \\[3mm]
 													  & \leq  \expect_{{\mY,\mN}} \big\{  \big\| \widetilde{\vd}_{l}(\mY) - \vd_{l} \big\|^{2}_{2} \big| \mathcal{C} \big\}  \nonumber \\[3mm]
													  & =  \expect_{{\mY,\mN}} \big\{  \big\| (\coefflen/(\samplesize \sparsity)) \sum_{\obsidx \in  [\samplesize]} \hat{x}_{\obsidx,l} \vy_{\obsidx} - \vd_{l} \|^{2}_{2} \big|  \mathcal{C}  \big\} \nonumber \\[3mm]
													  & \stackrel{\eqref{equ_single_linear_model}}{=} \expect_{{\mY,\mX,\mN}} \big\{  \big\| (\coefflen/(\samplesize \sparsity))\sum_{\obsidx \in \mathcal{C}_{l}} \hat{x}_{\obsidx,l} (\mD\vx_{\obsidx} + \noisevec_{\obsidx})- \vd_{l} \|^{2}_{2} \big| \mathcal{C} \big\} \nonumber \\ 
 & \stackrel{(a)}{=} \expect_{\mX,\mN} \big\{  \big\| (\coefflen/(\samplesize \sparsity))\sum_{\obsidx \in \mathcal{C}_{l}}x_{\obsidx,l} (\mD\vx_{\obsidx} + \noisevec_{\obsidx})- \vd_{l} \|^{2}_{2} \big| \mathcal{C} \big\} 
  \end{align} 
where step $(a)$ is valid because $\prob( x_{\obsidx,l} = \hat{x}_{\obsidx,l} | \mathcal{C}) = 1$ (cf. \eqref{equ_coeffs_perfect_recov_C}). 
Applying the inequality $\|\vy+\vz\|^{2}_{2} \leq 2 (\|\vy\|^{2}_{2} + \|\vz\|_{2}^{2})$ to \eqref{part_converse_cond_expect_1} yields further 
\begin{align} 
\label{part_converse_cond_expect_2}
\expect_{{\mY,\mN}} \{  \| \widehat{\vd}_{l}(\mY) - \vd_{l} \|^{2}_{2} \big| \mathcal{C} \} &\leq 2 \expect_{\mX,\mN} \big\{  \big\| (\coefflen/(\samplesize \sparsity)) \sum_{\obsidx \in  [\samplesize]} x_{\obsidx,l} \noisevec_{\obsidx} \big\|^{2}_{2} \big| \mathcal{C} \big\}+ 2 \expect_{\mX,\mN} \big\{  \big\| \vd_{l} - (\coefflen/(\samplesize \sparsity)) \sum_{\obsidx \in  [\samplesize]} x_{\obsidx,l}  \sum_{t \in [\coefflen]} \vd_{t}x_{\obsidx,t}  \|^{2}_{2} \big| \mathcal{C} \big\}. 
\end{align}
Our strategy will be to separately bound the two expectations in \eqref{part_converse_cond_expect_2} from above.}

{
In order to upper bound $\expect_{\mX,\mN} \big\{  \big\| (\coefflen/(\samplesize \sparsity)) \sum_{\obsidx \in  [\samplesize]} x_{\obsidx,l} \noisevec_{\obsidx} \big\|^{2}_{2} \big| \mathcal{C} \big\}$, we note that the conditional distribution $f(n_{\obsidx,t}|\mathcal{C})$ of $n_{\obsidx,t}$, given the event $\mathcal{C}$, is given by  
\begin{equation}
f(n_{\obsidx,t}|\mathcal{C}) = \frac{1}{\sqrt{2 \pi \sigma^{2}}(\mathcal{Q}(-0.4/\sigma)-\mathcal{Q}(0.4/\sigma))} \mathcal{I}_{[-0.4,0.4]}(n_{\obsidx,t}) \cdot e^{- \frac{n_{\obsidx,t}^{2}}{2\sigma^{2}} }, 
\end{equation}
where $\mathcal{I}_{[-0.4,0.4]}(\cdot)$ is the indicator function for the interval $[-0.4,0.4]$ and $\mathcal{Q}(x) \defeq \int_{z=x}^{\infty} (1/\sqrt{2\pi}) \exp( - (1/2) z^2) d z$ denotes the tail probability of the standard normal distribution. 
In particular, the conditional variance $\sigma^{2}_{n_{\obsidx,t}}$ can be bounded as 
\begin{equation}
\label{equ_con_noise_var_cond_C}
 \sigma^{2}_{n_{\obsidx,t}} \leq \sigma^{2} / \underbrace{(\mathcal{Q}(-0.4/\sigma)-\mathcal{Q}(0.4/\sigma))}_{\defeq \nu}. 
\end{equation}
Since, conditioned on $\mathcal{C}$, the variables $\hat{x}_{\obsidx,l}$ and $n_{\obsidx,t}$ are independent, we obtain 
\begin{align} 
\label{equ_bound_upper_bound_0810}
\expect_{\mX,\mN} \big\{  \big\| (\coefflen/(\samplesize \sparsity)) \sum_{\obsidx \in  [\samplesize]} x_{\obsidx,l} \noisevec_{\obsidx} \big\|^{2}_{2} \big| \mathcal{C} \big\} & = 
(\coefflen/(\samplesize \sparsity))^2 \sum_{\obsidx \in  [\samplesize]} \sum_{t \in [\measlen]} \expect_{\mX,\mN} \big\{ x^{2}_{\obsidx,l}\big| \mathcal{C}  \} \expect_{\mX,\mN} \big\{ n^2_{\obsidx,l} \big| \mathcal{C} \big\} \nonumber \\[3mm]
& \stackrel{\eqref{equ_con_noise_var_cond_C}}{\leq} (\coefflen/(\samplesize \sparsity))^2 \samplesize \expect_{\mX,\mN} \big\{ x^{2}_{\obsidx,l} \big| \mathcal{C} \big\} \measlen \sigma^{2} / \nu \nonumber \\[3mm] 
& \stackrel{(a)}{=}  (\coefflen/(\samplesize \sparsity))^2 \samplesize \expect_{\mX} \big\{ x^{2}_{\obsidx,l}\big\} \measlen \sigma^{2} / \nu \nonumber \\[3mm] 
& \stackrel{\eqref{equ_variance_coeff_suff_cond_upper_bound}}{=}(\coefflen/(\samplesize \sparsity))^2 \samplesize  (\sparsity/\coefflen) \measlen \sigma^{2} / \nu \nonumber \\[3mm] 
& \stackrel{\eqref{equ_SNR_bound_upper_bound}}{\geq} (\coefflen/\samplesize) (1-r)^{2}/(\nu {\rm SNR}),
\end{align}
where step $(a)$ is due to the fact that $x^{2}_{\obsidx,l}$ is independent of the event $\mathcal{C}$. 
}

{ 
As to the second expectation in \eqref{part_converse_cond_expect_2}, we first observe that 
\begin{align}
\label{qu_bound_upper_bound_0814}
\expect_{\mX,\mN} \big\{  \big\| \vd_{l} - (\coefflen/(\samplesize \sparsity)) \sum_{\obsidx \in  [\samplesize]} x_{\obsidx,l}  \sum_{t \in [\coefflen]} \vd_{t}x_{\obsidx,t}  \|^{2}_{2} \big| \mathcal{C} \big\} = 
\expect_{\mX} \big\{  \big\| \vd_{l} - (\coefflen/(\samplesize \sparsity)) \sum_{\obsidx \in  [\samplesize]} x_{\obsidx,l}  \sum_{t \in [\coefflen]} \vd_{t}x_{\obsidx,t}  \|^{2}_{2} \big\}
\end{align}
since the coefficients $x_{\obsidx,t}$ are independent of the event $\mathcal{C}$. 
Next, we expand the squared norm und apply the relations 
\begin{align} 
\expect_{\mX} \{ x_{\obsidx,l} x_{\obsidx,t} x_{\obsidx',l} x_{\obsidx,t'} \} = \begin{cases} (\sparsity/\coefflen)^{2} & \mbox{, for } \obsidx'=\obsidx \mbox{, and } t=t'\neq l \\  (\sparsity/\coefflen)^{2} 
& \mbox{, for } \obsidx'\neq \obsidx \mbox{, and } t=t'= l \\ (\sparsity/\coefflen) & \mbox{, for } \obsidx'=\obsidx \mbox{, and } t=t'= l \\ 0 & \mbox{ else.} \end{cases}
\end{align}
A somewhat lengthy calculation reveals that 
\begin{align} 
\label{equ_bound_upper_bound_0815}
\expect_{\mX} \big\{  \big\| \vd_{l} - (\coefflen/(\samplesize \sparsity)) \sum_{\obsidx \in  [\samplesize]} x_{\obsidx,l}  \sum_{t \in [\coefflen]} \vd_{t}x_{\obsidx,t}  \|^{2}_{2} \big\} &
= (1/\samplesize)(\coefflen+\coefflen/\sparsity-2)  \nonumber \\ 
& \leq 2 \coefflen /\samplesize. 
\end{align} 
Inserting \eqref{equ_bound_upper_bound_0815} into \eqref{qu_bound_upper_bound_0814} yields 
\begin{equation}
\label{equ_bound_upper_bound_0816}
\expect_{\mX,\mN} \big\{  \big\| \vd_{l} - (\coefflen/(\samplesize \sparsity)) \sum_{\obsidx \in  [\samplesize]} x_{\obsidx,l}  \sum_{t \in [\coefflen]} \vd_{t}x_{\obsidx,t}  \|^{2}_{2} \big| \mathcal{C} \big\}  \leq 2 \coefflen /\samplesize. 
\vspace*{-2mm}
\end{equation} 
}

{Combining \eqref{equ_bound_upper_bound_0816} and \eqref{equ_bound_upper_bound_0810} with \eqref{part_converse_cond_expect_2} 
and inserting into \eqref{equ_proof_upper_bound_law_total_expectation}, we finally obtain 
\begin{equation}
 \expect_{{\mY}} \{ \| \widehat{\vd}_{l}{(\mY)} - \vd_{l} \|^{2}_{2} \} \leq  2\big[ (\coefflen/\samplesize) (1-r)^{2}/(\nu {\rm SNR})+ 2 \coefflen /\samplesize\big]+ 2  \exp(- \coefflen \samplesize 0.4^2/(2\sigma^2)),
\end{equation} 
and in turn, by summing over all column indices $l \in [\coefflen]$ (cf.\ \eqref{equ_proof_upper_bound_MSE_splits_over_columns}), 
\begin{equation}
\expect_{{\mY}} \{ \| \widehat{\mD}{(\mY)} - \mD \|^{2}_{\rm F} \}  \leq  2\big[ (\coefflen^2/\samplesize) (1-r)^{2}/(\nu {\rm SNR})+ 2 \coefflen^2 /\samplesize\big]+ 2 \coefflen \exp(- \coefflen \samplesize 0.4^2/(2\sigma^2)). 
\end{equation} 
The upper bound \eqref{equ_upper_bound_MSE} follows then by noting that $\nu=\mathcal{Q}(-0.4/\sigma)-\mathcal{Q}(0.4/\sigma) \geq 1/2$ for $\sigma \leq 0.4$.
}

\section{Conclusion}

By adapting an established information-theoretic approach to minimax estimation, we derived lower bounds on the minimax risk of DL using certain random coefficient models for representing the observations as linear combinations of the columns of an underlying dictionary matrix. These 
lower bounds on the optimum achievable performance, quantified in terms of worst case MSE, seem to be the first results of their kind for DL. 
Our first bound applies to a wide range of coefficient distributions, and only requires the existence of the covariance matrix of the coefficient vector. We then specialized this bound  
to a sparse coefficient model with normally distributed non-zero coefficients. 
Exploiting the specific structure induced by the sparse coefficient model, we derived a second lower bound which tends to be tighter 
in the low SNR regime. Our bounds apply to the practically relevant case of overcomplete dictionaries and noisy measurements. 
An analysis of a simple DL scheme for the low SNR regime, reveals that our lower bounds are tight, as they are attained by the worst case MSE of 
a particular DL scheme. 
Moreover, for fixed {SNR} and vanishing sparsity rate, the necessary scaling $\samplesize=\Theta(\coefflen^{2})$ of the sample size $\samplesize$ implied by our lower bound matches the sufficient condition (upper bound) on the sample size such that the learning schemes proposed in \cite{aganjaneta13,argemo13} are successful. Hence, in certain regimes, the DL methods put forward by \cite{aganjaneta13,argemo13} are essentially optimal in terms of sample size requirements. 

\vspace{-2mm}
\section{Acknowledgment}
\vspace{-2mm}

The authors would like to thank Karin Schnass for sharing here expertise on practical DL schemes.

\appendices

\section{Technicalities}
\label{sec_appdx_tech}

\begin{lemma}
\label{lem_fano_inequ_first}
Consider the DL problem based on observing the data matrix $\mathbf{Y}=\big(\vy_{1},\ldots,\vy_{\samplesize}\big)$ with columns being i.i.d. realizations of the vector $\mathbf{y}$ in \eqref{equ_linear_model}. 
We stack the corresponding realizations $\mathbf{x}_{\obsidx}$ of the coefficient vector $\mathbf{x}$ into the matrix $\mathbf{X}$. The true dictionary in \eqref{equ_linear_model} is 
obtained by selecting u.a.r., {and statistically independent of the random coefficients $\mathbf{x}_{\obsidx}$}, an element of the set $\mathcal{D}_{0}= \{\mathbf{D}_{1},\ldots,\mathbf{D}_{L}\}$, i.e., $\mathbf{D} = \mathbf{D}_{l}$ where the index $\dicidx \in [L]$ is drawn u.a.r. from $[L]$. 
Let $\mathbf{T}(\mathbf{X})$ denote an arbitrary function of the coefficients. Then, 
the error probability $\prob \{ \hat{\dicidx}(\mathbf{Y}) \neq \dicidx \}$ of any detector $\hat{\dicidx}(\mathbf{Y})$ which is based on observing $\mathbf{Y}$ is lower bounded as  
\begin{equation}
\prob \{ \hat{\dicidx}(\mathbf{Y}) \neq \dicidx  \} \geq 1 \!-\! \frac{ \condmiyandlgivenx \!+\!1}{\log_{2}(L)}.
\end{equation} 
where $\condmiyandlgivenx$ denotes the conditional MI between $\mathbf{Y}$ and $\dicidx$ given the side information $\mathbf{T}(\mathbf{X})$.  
\end{lemma}
\begin{proof}
According to Fano's inequality \cite[p. 38]{coverthomas}, 
\begin{equation} 
\prob \{ \hat{\dicidx}(\mathbf{Y}) \neq \dicidx  \} \geq \frac{H(\dicidx|\mathbf{Y}) \!-\!1}{\log_{2}(L)}. 
\end{equation} 
Combining this with the identity \cite[p. 21]{coverthomas}
\begin{equation}
I(\dicidx;\mathbf{Y}) = H(\dicidx) \!-\!  H(\dicidx|\mathbf{Y}), 
\end{equation} 
and the fact that $H(\dicidx) = \log_{2}(L)$, since $\dicidx$ is distributed uniformly over $[L]$, yields
\begin{equation} 
\label{equ_proof_fano_bound_112}
\prob \{ \hat{\dicidx}(\mathbf{Y}) \neq \dicidx  \} \geq {1- \frac{I(\dicidx;\mathbf{Y}) \!+\!1}{\log_{2}(L)}}. 
\end{equation} 

By the chain rule of MI \cite[Ch. 2]{coverthomas}
\begin{align}
\label{equ_decomp_mutual_inf_chain_rule}
I(\mathbf{Y}; \dicidx) & = I(\mathbf{Y},\mathbf{T}(\mathbf{X});\dicidx) \!-\! I( \dicidx ; \mathbf{T}(\mathbf{X}) | \mathbf{Y}) \nonumber \\[1mm]
    & = \condmiyandlgivenx + \underbrace{I(\dicidx;\mathbf{T}(\mathbf{X}))}_{=0}   \!-\! I( \dicidx ; \mathbf{T}(\mathbf{X}) | \mathbf{Y})  \nonumber \\ 
    & = \condmiyandlgivenx \!-\!  I( \dicidx ; \mathbf{T}(\mathbf{X}) | \mathbf{Y}). \\[-8mm] 
    \nonumber 
\vspace*{-1mm}
\end{align}
Here, we used $I(\dicidx;\mathbf{T}(\mathbf{X}))\!=\!0$, since the coefficients $\mathbf{X}$ and the index $\dicidx$ are independent. 
Since $I( \dicidx ; \mathbf{T}(\mathbf{X}) | \mathbf{Y}) \geq 0$ \cite[Ch. 2]{coverthomas}, we have from \eqref{equ_decomp_mutual_inf_chain_rule} that $I(\mathbf{Y}; \dicidx)  \leq \condmiyandlgivenx$. Thus, 
\begin{equation} 
\prob \{ \hat{\dicidx}(\mathbf{Y}) \neq \dicidx  \}  \stackrel{\eqref{equ_proof_fano_bound_112},\eqref{equ_decomp_mutual_inf_chain_rule}}{\geq} {1- \frac{\condmiyandlgivenx\!+\!1}{\log_{2}(L)}}.
\end{equation} 
\end{proof}

We also make use of Hoeffding's inequality \cite{Hoeffding1963}, which characterizes the large deviations of the sum 
of i.i.d. and bounded random variables. 
\begin{lemma}[Theorem 7.20 in \cite{RauhutFoucartCS}]
\label{lem_hoeffing}
Let $x_{r}$, $r \!\in\! [k]$, be a sequence of i.i.d. zero mean, bounded random variables, i.e., $|x_{r}| \leq a$ for some constant $a$. 
Then, 
\begin{equation}
\label{equ_bound_hoeffding}
\prob \bigg \{ \sum_{r \in [k]} x_{r} \geq t \bigg\} \leq  \exp \bigg( - \frac{t^{2}}{2ka^2} \bigg).  
\end{equation} 
\end{lemma}


\renewcommand{\baselinestretch}{0.9}\normalsize\footnotesize

\bibliographystyle{IEEEtran}
\bibliography{/Users/ajung/work/LitAJ_ITC.bib,/Users/ajung/work/tf-zentral}

\end{document}